\documentclass[11pt]{article} 
\usepackage{tikz}
\usepackage{blindtext}
\usepackage{xcolor}
\usepackage{siunitx}

\usepackage{amsmath,amsfonts,amssymb,amsthm,mathtools,mathrsfs}
\usepackage{thmtools}
\usepackage{thm-restate}

\usepackage{algorithm}
\usepackage{algorithmic}

\usepackage{bm}
\usepackage{graphicx}
\usepackage{multirow,booktabs,longtable,authblk}
\usepackage{float}
\usepackage{subcaption}
\usepackage{enumerate}
\usepackage{fancyhdr}
\usepackage{hhline}

\usepackage[top=2.5cm, bottom=2.5cm, left=3cm, right=3cm]{geometry}
\usepackage[numbers,sort&compress]{natbib}
\usepackage{hyperref}
\usepackage[numbers,sort&compress]{natbib} 

\newtheorem{assumption}{Assumption}

\newtheorem{definition}{Definition}

\newtheorem{lemma}{Lemma}
\newtheorem{proposition}{Proposition}
\newtheorem{remark}{Remark}

\allowdisplaybreaks[4]

\title{A Quantitative Approximation Framework for Flow Distillation in Diffusion Models\(^\dag\)\footnotetext{\dag~All authors contributed equally to this work and are listed alphabetically. The corresponding author is Hanfei Zhou. The work of Weiguo Gao is partially supported by the National Key R\&D Program of China (Grant No. 2021YFA1003305). The work of Lei Shi is partially supported by the National Natural Science Foundation of China (Grant No. 12571099).}}

\author[1,2]{Weiguo Gao}
\author[1,2]{Ming Li}
\author[1,2]{Lei Shi}
\author[1]{Hanfei Zhou}

\affil[1]{School of Mathematical Sciences, \linebreak
Fudan University, Shanghai, 200433, China 
}
\affil[2]{
Shanghai Key Laboratory of Contemporary Applied Mathematics, \linebreak
Fudan University, Shanghai, 200433, China \linebreak
Email: wggao@fudan.edu.cn, mingli23@m.fudan.edu.cn, \linebreak
leishi@fudan.edu.cn, zhouhf23@m.fudan.edu.cn
}
\date{}
\begin{document}
\maketitle

\begin{abstract}
We develop a quantitative framework for diffusion distillation by viewing few step sampling as approximation through compositions of learned flow maps. For trajectory distillation of the probability flow ODE, we show that low noise multimodal regimes separate score approximability from dynamical stability: the score remains efficiently approximable, while small local errors may be strongly amplified by stiff flow dynamics. In a Gaussian mixture Ornstein--Uhlenbeck model, we prove time uniform \(L^p(p_t)\) score approximation by ReLU and ReQU networks with explicit polylogarithmic complexity, and derive a computable Lipschitz bound \(L(t)\) for the flow velocity. The stability factor \(\exp\bigl(\int_s^t L(u)\mathrm du\bigr)\) can grow exponentially as noise decreases and mixture separation increases. Comparing this certificate with a certified local Lipschitz budget for one step students identifies regimes of direct distillation difficulty, without implying an approximation lower bound. We also show that deep residual compositions control global transport error through propagated local errors, and that equalizing cumulative stability yields an optimal nonuniform segmentation. With eight segments, this grid reduces final mean relative MSE by up to \(51.9\%\) versus uniform grids.
\end{abstract}

{\textbf{Keywords and phrases:} Diffusion Model; Trajectory Distillation; Probability Flow ODE; Approximation Theory; Gaussian Mixture Model.}

\section{Introduction}
\label{sec:introduction}

This paper studies diffusion distillation with deep neural networks from the perspective of approximation theory. The goal is to compress a long-horizon generative transport into a small number of learned maps while preserving sample quality. Diffusion models have become a leading paradigm for generative modeling, enabling high-fidelity synthesis across modalities and finding increasing use in scientific and engineering applications~\cite{gao2026terminally,ho2020denoising,song2021sde}. In their continuous-time formulation~\cite{song2021sde}, a forward stochastic process transports a complex data distribution to a simple Gaussian prior and induces a family of marginals \(\{p_t\}_{t\in[0,T]}\). Generation then amounts to reversing this transport through either the reverse-time stochastic differential equation (SDE) or the deterministic probability flow ordinary differential equation (ODE), which share the same marginal laws. This formulation reduces learning to estimating the score \(\nabla_{\bm x}\log p_t(\bm x)\), or equivalently a velocity field, and reduces sampling to numerical integration of the resulting dynamics. Thus, generation quality depends both on the accuracy of the learned vector field and on the way approximation errors propagate along the long-horizon flow.

The need for long-horizon numerical integration makes inference with diffusion models expensive, and has motivated substantial work on accelerating sampling. Early approaches mainly focused on fast solvers that reuse a pretrained model and reduce the number of function evaluations through improved discretizations or higher-order ODE schemes, including DDIM~\cite{song2020denoising}, DPM-Solver~\cite{lu2022dpmsolver}, UniPC~\cite{zhao2023unipc}, and EDM-style sampling designs~\cite{karras2022edm}. These methods can substantially reduce computation, but they still require multiple model evaluations and their sample quality may degrade when the step budget becomes very small. More recently, distillation has become a widely used alternative for accelerating diffusion sampling. In particular, \emph{trajectory distillation} compresses a pretrained sampler by learning a small number of student updates that imitate the teacher's multistep reverse-time denoising trajectory, as in progressive distillation~\cite{salimans2022progressive}, consistency models~\cite{song2023consistency}, and data-free variants such as BOOT~\cite{gu2023boot}. Other distillation strategies, including distribution matching~\cite{zhu2025simple,yin2024one,yin2024improved} and adversarial distillation~\cite{sauer2024adversarial}, instead emphasize global distributional alignment or perceptual objectives and need not reproduce the teacher dynamics. In parallel, another line of work trains generic one-step or few-step generators from scratch, such as mean-flow models~\cite{geng2025mean} and shortcut diffusion models~\cite{frans2025one}. We focus on trajectory distillation because it retains the teacher as a reference transport and turns few-step sampling into an error-propagation problem: local discrepancies can be amplified by the Gr\"onwall factor associated with the Jacobian of the reference flow, especially in low-noise, multimodal regimes where the dynamics become stiff.

Despite rapid empirical progress, the theoretical understanding of diffusion distillation remains limited. In particular, it is still unclear how to quantify the error incurred when a teacher's long-horizon transport is replaced by only a small number of learned maps, and how these errors accumulate under composition. A theoretical viewpoint is therefore useful for clarifying the limitations of few-step students and for guiding principled algorithm and time-grid design. We now formalize this perspective in terms of probability-flow maps.

Consider a forward SDE on \(\mathbb R^d\) over \(t\in[0,T]\),
\begin{equation*}
\mathrm d \bm X_t = \bm f_t(\bm X_t)\mathrm dt + \bm g_t(\bm X_t)\mathrm d\bm B_t,
\quad
\bm X_0 \sim p_0,
\end{equation*}
and let \(p_t\) denote the marginal distribution of \(\bm X_t\) for \(t\in[0,T]\). Under standard assumptions ensuring existence of the reverse-time dynamics, there is an associated probability flow ODE on \([0,T]\) whose time-\(t\) marginal coincides with \(p_t\) for all \(t\in[0,T]\)~\cite{song2021sde}. In the variance-preserving (VP) setting, this ODE can be written as
\begin{equation*}
\dot{\bm x}_t = \bm v(\bm x_t,t),\quad t\in[0,T],
\quad
\bm v(\bm x,t) = -\beta_t \bigl(\bm x + \nabla_{\bm x}\log p_t(\bm x)\bigr).
\end{equation*}
For any \(0\le s < t \le T\), let \(\Phi_{s\leftarrow t}\) denote the corresponding flow map, meaning that if \(\bm x_\cdot\) solves the ODE then \(\Phi_{s\leftarrow t}(\bm x_t)=\bm x_s\). A common approach to distillation is to introduce a coarse time grid \(T=t_0>t_1>\dots>t_n=0\) and learn a small collection of student maps \(\{\Psi_{t_k\leftarrow t_{k-1}}\}_{k=1}^n\) that approximate the corresponding teacher flow segments \(\{\Phi_{t_k\leftarrow t_{k-1}}\}_{k=1}^n\). The resulting distilled generator is the composed map
\begin{equation*}
\Psi_{0\leftarrow T}
\coloneqq
\Psi_{t_n\leftarrow t_{n-1}}\circ\cdots\circ\Psi_{t_1\leftarrow t_0}
\approx
\Phi_{t_n\leftarrow t_{n-1}}\circ\cdots\circ\Phi_{t_1\leftarrow t_0}
=\Phi_{0\leftarrow T}.
\end{equation*}
This decomposition leads to the following analysis strategy. We first quantify the network complexity required to approximate the time-dependent score, or equivalently the probability-flow velocity, associated with the marginals \(\{p_t\}_{t\in[0,T]}\), and track how this complexity depends on the geometry of the mixture distribution. We then combine these approximation estimates with stability bounds for the induced flow to obtain efficient neural-network approximations of finite-time flow maps \(\Phi_{s\leftarrow t}\), including explicit depth and width scalings in terms of the target accuracy and the stability quantities controlling error amplification. This viewpoint yields a segment-wise measure of distillation difficulty and leads to a non-uniform time discretization that balances this difficulty over the full horizon.

To keep the analysis tractable without removing the mechanism responsible for multimodal stiffness, we specialize to an isotropic Gaussian mixture model (GMM) initialization evolved by an Ornstein--Uhlenbeck (OU) forward process. This specialization does not trivialize the distillation problem; rather, it gives a closed-form setting in which the two relevant sources of difficulty can be isolated. In this GMM--OU setting, each marginal \(p_t\) remains a Gaussian mixture, and hence both the score \(\nabla_{\bm x}\log p_t(\bm x)\) and the Hessian \(\nabla_{\bm x}^2\log p_t(\bm x)\) have explicit formulas. These formulas yield computable bounds on the Jacobian of the probability-flow velocity field, which in turn give flow map stability estimates through Gr\"onwall arguments and quantify the amplification of perturbations over time. The same closed-form structure separates two effects. At each fixed time, score approximation remains a nontrivial neural approximation problem whose complexity depends on the mixture geometry. Across time, local approximation errors are propagated by the flow and may accumulate into large trajectory errors when the integrated stability bound is large. This separation underlies the subsequent theory and motivates the stability-balanced discretization strategy used in multistage distillation.

\paragraph{Contributions.}
The main contributions of this paper are summarized as follows.
\begin{itemize}
\item We prove a constructive \(L^p(p_t)\) approximation theorem for the GMM score by ReLU--ReQU networks, uniformly over the marginal family \(\{p_t\}_{t\in[0,T]}\). In particular, for any \(0<\varepsilon<1/2\), we construct \(\hat{\bm s}_\varepsilon(\cdot,t)\in\mathcal F(Q,G,S,B)\) such that
\begin{equation*}
\bigl(\mathbb E_{\bm X_t\sim p_t}\bigl[\lVert \hat{\bm s}_\varepsilon(\bm X_t,t)-\nabla_{\bm x}\log p_t(\bm X_t)\rVert^p\bigr]\bigr)^{1/p}
\le C\varepsilon,
\end{equation*}
with depth and width scaling as \(Q\le C_1\bigl(\log K+(1+\log N_{\varepsilon,\sigma})^2\bigr)\) and \(G\le C_2K\bigl(d+(1+\log N_{\varepsilon,\sigma})^3\bigr)\), where \(N_{\varepsilon,\sigma}=2+d+K+\varepsilon^{-1}+\sigma+\sigma^{-1}\).

\item Using the closed-form Hessian of the marginal density, we derive an explicit and computable bound \(L(t)\) for the spatial Lipschitz constant of the probability-flow velocity. This yields the flow map stability estimate
\begin{equation*}
\lVert\Phi_{s\leftarrow t}(\bm x)-\Phi_{s\leftarrow t}(\bm y)\rVert_2
\le
\exp\Bigl(\int_s^t L(\tau)\mathrm d\tau\Bigr)\lVert\bm x-\bm y\rVert_2.
\end{equation*}
As made explicit in~\autoref{prop:flow_stability}, the stability exponent depends on both the length of the time interval and the multimodal structure of the marginal distribution. This gives a quantitative way to compare the difficulty of different flow segments and explains why low-noise multimodal regimes are unfavorable for long-interval distillation.

\item We identify a stability-certificate mismatch criterion for one-step distillation. In low-noise multimodal regimes, the computed tail Gr\"onwall factor can exceed a certified local Lipschitz budget for a one-step student with the prescribed architecture. This is a comparison of two a priori certificates, not an \(L^p\)-approximation lower bound. Motivated by this comparison, we prove that the dynamical target \(\Phi_{0\leftarrow T}\) can be approximated by a clipped deep residual composition \(\mathcal R_{n,R}(\mathcal F)\) that follows the semigroup structure of the exact flow. In particular, for every fixed horizon \(T\), \(n=\widetilde{\mathcal O}_{T,\mathrm{model}}(\varepsilon^{-1})\) residual blocks suffice, the per-block complexity matches the compact fixed-time score approximation scale, and the global error is controlled up to the stability amplification factor.

\item We convert the Jacobian-based stability bound into a non-uniform time discretization rule by equalizing the segment-wise stability exponents \(\Lambda_k=\int_{t_k}^{t_{k-1}}L(u)\mathrm du\). Equivalently, with \(A(t)=\int_0^t L(u)\mathrm du\), the decreasing generation grid is chosen by the uniform partition \(A(t_k)=\frac{n-k}{n}A(T)\). Among all partitions with \(n\) segments, this construction minimizes the largest segmentwise exponent \(\max_k\Lambda_k\), and hence the largest segmentwise Gr\"onwall certificate. The resulting stability-balanced grid refines the time discretization where the certified amplification is largest and provides an explicit, model-driven strategy for multistage distillation.
\end{itemize}

\paragraph{Organization.}
The remainder of the paper is organized as follows. \autoref{sec:theoretical_setup} introduces the GMM--OU setting and derives closed-form formulas for the score, the Hessian, and the stability bound of the probability-flow ODE. \autoref{sec:approximation_theory_for_score_learning} studies score approximation by ReLU--ReQU networks and quantifies how the required network complexity depends on the mixture parameters and the target accuracy. \autoref{sec:distillation_and_stability_guided_time_discretization} uses the Jacobian-based stability bound to identify a stability-certificate mismatch criterion for one-step distillation, proves an efficient approximation result for finite-time flow maps by deep residual compositions, and develops a stability-balanced time discretization rule. \autoref{sec:experiments} provides empirical illustrations, and the appendices contain proofs and supporting technical results. Detailed proofs of the main theorems and auxiliary lemmas are provided in the~\hyperref[allapp]{Appendix}.

\section{Related Work}
\label{sec:related}

\paragraph{Theoretical studies of diffusion models.}
One line of theory studies diffusion models under general distributional assumptions and establishes approximation, statistical, and computational guarantees. Approximation results show that deep networks can approximate score functions while avoiding the curse of dimensionality under suitable structural assumptions~\cite{de2022convergence,oko2023diffusion,tang2024adaptivity,yakovlev2025simultaneous}. Statistical analyses further show that the sample complexity can depend on the intrinsic dimension of the data distribution rather than the ambient dimension~\cite{chen2023score,chen2023sampling,lee2022convergence}. Other works connect score matching with neural tangent kernels to obtain generalization bounds~\cite{han2024neural}, or provide evaluation-complexity guarantees for higher-order training-free sampling schemes~\cite{li2025faster}.

A second line of work studies diffusion models in analytically tractable regimes, with Gaussian mixtures serving as a canonical model class for deriving sharper guarantees. In this setting, prior work has studied parameter identifiability and recovery~\cite{gatmiry2024learning,shah2023learning}, solver robustness~\cite{jia2024structured,liang2025unraveling,guo2023gaussian}, and mechanisms for exact inversion or polynomial approximation~\cite{zhang2024exact,wang2026error}. The same mixture framework has also been used to characterize distributional dynamics and sampling limits, including critical windows for feature emergence~\cite{li2024critical}, sample-complexity estimates~\cite{cui2023analysis}, and discretization bounds~\cite{wang2024diffusion}. It has further enabled analyses of training objectives, including limitations of guidance~\cite{chidambaram2024does}, memorization trade-offs~\cite{buchanan2025edge,gao2024flow,zhou2026smoothing}, and velocity field matching~\cite{vuong2025we}. These works provide important insights into diffusion dynamics, but they do not quantify the composition error that arises when a continuous-time flow is replaced by a short sequence of learned maps, which is the central operation in trajectory distillation.

\paragraph{Flow-map and trajectory distillation.}
Recent work has begun to learn maps directly between arbitrary noise levels. The work in~\cite{sabour2025align} introduces continuous-time flow-map distillation and studies how learned maps scale from one-step to multistep generation. The method in~\cite{tong2025flowmap} proposes data-free flow-map distillation and includes a correction mechanism for compounding trajectory errors. These methods emphasize scalable learning objectives and empirical generation quality. Our contribution is complementary: in a closed-form GMM--OU setting, we derive constructive network-complexity bounds, a computable certificate-level analysis of dynamical amplification under composition, and a minimax rule for balancing the largest segmentwise stability certificate.

Our work provides a quantitative approximation analysis of composition in the GMM--OU setting. We formulate distillation as flow-map approximation under composition and derive bounds that separate segment-level approximation error from dynamical amplification. Although the forward dynamics are linear, the multimodal initialization induces a nonlinear score field in which posterior mixture weights can change rapidly, capturing stiffness and perturbation sensitivity in low-noise regimes. Closed-form expressions for the score and its Jacobian yield computable bounds on the time-dependent Jacobian of the probability-flow ODE. This analysis identifies regimes in which the standard stability certificate becomes unfavorable for few-step students even when the score is accurately approximated at fixed times, and it motivates a grid that balances the largest certified segmentwise amplification.

\paragraph{Stability, stiffness, and Lipschitz constraints.}
A related line of work studies stability and smoothness properties of diffusion dynamics, since stiff vector fields and large Lipschitz constants can hinder both training and sampling~\cite{zhang2024stability,yang2023lipschitz}. This regularity viewpoint has led to end-time scaling and Lipschitz-guided schedule design for mitigating singular behavior and mode collapse near the origin~\cite{yang2023lipschitz,chen2025lipschitz}. Within Gaussian mixture settings, researchers have also derived component-independent Lipschitz bounds and designed moment-matching solvers with rigorous sampling stability guarantees~\cite{liang2025unraveling,guo2023gaussian}.

Our theory connects this stability viewpoint directly to multistage distillation. We show that the difficulty of a flow segment is determined not only by fixed-time score approximation, but also by the stiffness of the induced flow and its sensitivity to perturbations. This reveals regimes in which the score field is comparatively easy to approximate while the composed transport remains highly unstable. Based on the resulting time-dependent Jacobian bound, we derive a stability-balanced discretization rule that refines the time grid where the certificate is large and minimizes the largest segmentwise certified amplification. This certificate-level property does not by itself imply minimization of the exact end-to-end error.

\section{Theoretical Setup}
\label{sec:theoretical_setup}

This section follows the modeling framework in~\cite{de2022convergence} and provides the technical setup for the later analysis. We construct an analytically tractable setting in which the main quantities needed below can be written in closed form. Specifically, we consider a VP diffusion whose marginals \(\{p_t\}_{t\in[0,T]}\) are generated by an OU process~\cite{song2021sde}, initialized from an isotropic GMM. In this setting, we derive closed-form expressions for the score \(\nabla_{\bm x}\log p_t\) and the Hessian \(\nabla_{\bm x}^2\log p_t\), and then use these formulas to obtain explicit Lipschitz and flow-stability bounds for the associated probability-flow ODE.

\subsection{OU forward process and GMM initialization}
\label{subsec:ou_forward_process_and_gmm_initialization}

We adopt the standard VP forward noising process used in score-based diffusion models~\cite{ho2020denoising,song2021sde}, which can be written as an OU process on \(\mathbb R^d\):
\begin{equation}
\label{eq:ou_sde}
\mathrm d\bm X_t=-\beta_t \bm X_t\mathrm dt+\sqrt{2\beta_t}\mathrm d\bm B_t,\quad t\in[0,T].
\end{equation}
Here \(\bm B_t\) is a standard \(d\)-dimensional Brownian motion and \(\beta_t\geq0\) is a possibly time-dependent noise schedule. Define
\begin{equation*}
m_t \coloneqq \exp\Bigl(-\int_0^t \beta_s\mathrm ds\Bigr),\quad \sigma_t^2 \coloneqq 1-m_t^2 = 1-\exp\Bigl(-2\int_0^t \beta_s\mathrm ds\Bigr).
\end{equation*}
The solution of~\eqref{eq:ou_sde} then admits the explicit representation
\begin{equation*}
\bm X_t=m_t\bm X_0+\sigma_t\bm\varepsilon,\quad \bm\varepsilon\sim\mathcal N(\bm 0,\bm I_d),\quad \bm\varepsilon\ \text{independent of }\bm X_0.
\end{equation*}
Equivalently, the conditional law of \(\bm X_t\) given \(\bm X_0\) is
\begin{equation}
\label{eq:ou_conditional}
\bm X_t\mid\bm X_0\sim\mathcal N\bigl(m_t\bm X_0,\sigma_t^2\bm I_d\bigr).
\end{equation}
To keep the forward marginals analytically tractable while retaining multimodality, we initialize the OU process from an isotropic Gaussian mixture, as specified in~\autoref{ass:gmm_ou_init}. Since the OU transition kernel is Gaussian, each initial mixture component remains Gaussian under the forward evolution, and the marginal \(p_t\) stays within the Gaussian-mixture family. This closed-form structure allows us to derive explicit formulas for the score \(\nabla_{\bm x}\log p_t\) and the Hessian \(\nabla_{\bm x}^2\log p_t\), which are the main ingredients in the flow-stability bounds developed below.

For the time-conditioned approximation theorem in~\autoref{sec:approximation_theory_for_score_learning}, we use the following regularity condition on the noise schedule. The closed-form identities in this section only require the quantities above to be well-defined.

\begin{assumption}
\label{ass:suzuki_schedule}
There are fixed constants \(0<\underline\beta\le\overline\beta<\infty\) and \(C_\beta\ge1\) such that
\begin{equation*}
\underline\beta\le\beta_t\le\overline\beta,
\qquad
\sup_{r\ge1}\sup_{t\in[0,T]}\bigl|\partial_t^r\beta_t\bigr|\le C_\beta.
\end{equation*}
This is the fixed-bound version of the schedule regularity in Assumption~2.5 in~\cite{oko2023diffusion}. Only the restriction to the compact interval \([0,T]\) is used below, and the proof of Lemma~B.1 in the same work applies verbatim on this interval with constants depending on \(\underline\beta,\overline\beta,C_\beta\), and \(T\). In particular, every constant schedule \(\beta_t\equiv\beta>0\) is included.
\end{assumption}

\begin{assumption}
\label{ass:gmm_ou_init}
We assume that the initial distribution \(p_0(\bm x)\) is a \(d\)-dimensional isotropic GMM, i.e.,
\begin{equation}
\label{eq:gmm_p0}
p_0(\bm x)=\sum_{k=1}^K \pi_k\mathcal N\bigl(\bm x;\bm \mu_k,\sigma^2 \bm I_d\bigr),\quad\text{where }
\sigma>0,
\quad
\pi_k>0,
\quad
\sum_{k=1}^K \pi_k=1.
\end{equation}
We write \(\pi_{\min}\coloneqq\min_{1\le k\le K}\pi_k\). Components with zero mixture weight may be removed before applying the results below.
\end{assumption}

By a slight abuse of notation, we use \(\mathcal{N}(\bm x;\bm \mu,\bm \Sigma)\) to denote the corresponding Gaussian density at \(\bm x\). Under~\autoref{ass:gmm_ou_init} and the transition kernel in~\eqref{eq:ou_conditional}, the marginal \(p_t(\bm x)\) remains a GMM for any \(t\in[0,T]\). Define
\begin{equation*}
s_t^2 \coloneqq m_t^2\sigma^2+\sigma_t^2,
\quad
\bm m_k(t)\coloneqq m_t\bm \mu_k.
\end{equation*}
Then
\begin{equation}
\label{eq:gmm_pt}
p_t(\bm x)
=
\sum_{k=1}^K \pi_k\mathcal N\bigl(\bm x;\bm m_k(t),s_t^2 \bm I_d\bigr).
\end{equation}
The same derivation extends to anisotropic mixtures by replacing \(\sigma^2 \bm I_d\) with \(\bm \Sigma_k\). For clarity, we focus on the isotropic setting in~\eqref{eq:gmm_p0}. We next provide closed-form expressions for the GMM score \(\nabla_{\bm x}\log p_t\) and Hessian \(\nabla_{\bm x}^2\log p_t\) in~\autoref{prop:gmm_score_hess_closed}.

\begin{proposition}
\label{prop:gmm_score_hess_closed}
For any \(t\in[0,T]\), let \(p_t\) be an isotropic GMM
\begin{equation*}
p_t(\bm x)=\sum_{k=1}^K \pi_k\phi_k(\bm x,t),
\quad
\phi_k(\bm x,t)\coloneqq\mathcal N\bigl(\bm x;\bm m_k(t),s_t^2 \bm I_d\bigr),
\end{equation*}
where \(\pi_k>0\) and \(\sum_{k=1}^K \pi_k=1\). Define the posterior mixture weights
\begin{equation*}
\gamma_k(\bm x,t)
\coloneqq
\frac{\pi_k\phi_k(\bm x,t)}{\sum_{j=1}^K \pi_j\phi_j(\bm x,t)},
\quad
\sum_{k=1}^K \gamma_k(\bm x,t)=1.
\end{equation*}
Define the conditional mean and covariance of the component centers under the weights \(\gamma_k(\bm x,t)\):
\begin{equation*}
\begin{aligned}
\bar{\bm m}(\bm x,t)&\coloneqq\sum_{k=1}^K \gamma_k(\bm x,t)\bm m_k(t),\\
\Sigma_\gamma(\bm x,t)
&\coloneqq\sum_{k=1}^K \gamma_k(\bm x,t)\bigl(\bm m_k(t)-\bar{\bm m}(\bm x,t)\bigr)\bigl(\bm m_k(t)-\bar{\bm m}(\bm x,t)\bigr)^\top.
\end{aligned}
\end{equation*}
Then the score function and Hessian satisfy
\begin{equation}
\label{eq:score_hess_display}
\nabla_{\bm x} \log p_t(\bm x)
=-\frac{1}{s_t^2}\bigl(\bm x-\bar{\bm m}(\bm x,t)\bigr),\quad
\nabla_{\bm x}^2 \log p_t(\bm x)
=-\frac{1}{s_t^2}\bm I_d+\frac{1}{s_t^4}\Sigma_\gamma(\bm x,t).
\end{equation}
\end{proposition}

\begin{proof}
Fix \(t\in[0,T]\) and \(\bm x\in\mathbb R^d\). For each component density
\(\phi_k(\bm x,t)=\mathcal N\bigl(\bm x;\bm m_k(t),s_t^2\bm I_d\bigr)\), we have
\begin{equation}
\label{eq:grad_log_phi_revised}
\nabla_{\bm x}\log \phi_k(\bm x,t)
=
-\frac{1}{s_t^2}\bigl(\bm x-\bm m_k(t)\bigr).
\end{equation}
Since \(p_t(\bm x)=\sum_{k=1}^K \pi_k\phi_k(\bm x,t)\), differentiating \(\log p_t(\bm x)\) gives
\begin{equation*}
\begin{aligned}
\nabla_{\bm x}\log p_t(\bm x)
&=
\frac{1}{p_t(\bm x)}\sum_{k=1}^K \pi_k\nabla_{\bm x}\phi_k(\bm x,t) =\sum_{k=1}^K \frac{\pi_k\phi_k(\bm x,t)}{\sum_{j=1}^K \pi_j\phi_j(\bm x,t)}\nabla_{\bm x}\log \phi_k(\bm x,t)\\
&=
\sum_{k=1}^K \gamma_k(\bm x,t)\nabla_{\bm x}\log \phi_k(\bm x,t)= -\frac{1}{s_t^2}\sum_{k=1}^K \gamma_k(\bm x,t)\bigl(\bm x-\bm m_k(t)\bigr)\\
&=
-\frac{1}{s_t^2}\Bigl(\bm x-\sum_{k=1}^K \gamma_k(\bm x,t)\bm m_k(t)\Bigr)=-\frac{1}{s_t^2}\bigl(\bm x-\bar{\bm m}(\bm x,t)\bigr),
\end{aligned}
\end{equation*}
which proves the score identity in~\eqref{eq:score_hess_display}.

We next compute the Hessian. Since \(\gamma_k(\bm x,t)=\pi_k\phi_k(\bm x,t)/p_t(\bm x)\), taking a log-gradient yields
\begin{equation}
\label{eq:grad_log_gamma_revised}
\nabla_{\bm x}\log \gamma_k(\bm x,t)
=
\nabla_{\bm x}\log \phi_k(\bm x,t)-\nabla_{\bm x}\log p_t(\bm x).
\end{equation}
Multiplying~\eqref{eq:grad_log_gamma_revised} by \(\gamma_k(\bm x,t)\) and substituting~\eqref{eq:grad_log_phi_revised} together with the score formula above gives
\begin{equation*}
\nabla_{\bm x}\gamma_k(\bm x,t)
=
\gamma_k(\bm x,t)\bigl(\nabla_{\bm x}\log \phi_k(\bm x,t)-\nabla_{\bm x}\log p_t(\bm x)\bigr)
=
\frac{\gamma_k(\bm x,t)}{s_t^2}\bigl(\bm m_k(t)-\bar{\bm m}(\bm x,t)\bigr).
\end{equation*}
Using \(\bar{\bm m}(\bm x,t)=\sum_{k=1}^K \gamma_k(\bm x,t)\bm m_k(t)\) and the fact that \(\bm m_k(t)\) does not depend on \(\bm x\), we obtain the Jacobian
\begin{equation*}
\begin{aligned}
\nabla_{\bm x}\bar{\bm m}(\bm x,t)
&=
\sum_{k=1}^K \bm m_k(t)\nabla_{\bm x}\gamma_k(\bm x,t)^\top
=\frac{1}{s_t^2}\sum_{k=1}^K \gamma_k(\bm x,t)\bm m_k(t)\bigl(\bm m_k(t)-\bar{\bm m}(\bm x,t)\bigr)^\top\\
&=
\frac{1}{s_t^2}\sum_{k=1}^K \gamma_k(\bm x,t)\bigl(\bm m_k(t)-\bar{\bm m}(\bm x,t)\bigr)\bigl(\bm m_k(t)-\bar{\bm m}(\bm x,t)\bigr)^\top\\
&=
\frac{1}{s_t^2}\Sigma_\gamma(\bm x,t),
\end{aligned}
\end{equation*}
where we used \(\sum_{k=1}^K \gamma_k(\bm x,t)\bigl(\bm m_k(t)-\bar{\bm m}(\bm x,t)\bigr)=\bm 0\) to pass to the third line.
Finally, we differentiate
\(\nabla_{\bm x}\log p_t(\bm x)=-s_t^{-2}\bigl(\bm x-\bar{\bm m}(\bm x,t)\bigr)\)
and obtain
\begin{equation*}
\begin{aligned}
\nabla_{\bm x}^2\log p_t(\bm x)
&=
-\frac{1}{s_t^2}\Bigl(\bm I_d-\nabla_{\bm x}\bar{\bm m}(\bm x,t)\Bigr) = -\frac{1}{s_t^2}\bm I_d+\frac{1}{s_t^4}\Sigma_\gamma(\bm x,t),
\end{aligned}
\end{equation*}
which is the Hessian identity in~\eqref{eq:score_hess_display}.
\end{proof}

\begin{remark}
The posterior mean
\[
\bar{\bm m}(\bm x,t)
=
\sum_{k=1}^K\gamma_k(\bm x,t)\bm m_k(t)
\]
has the form of a Nadaraya--Watson estimator with Gaussian kernels centered at the contracted means \(\bm m_k(t)\). The score itself is the scaled residual
\[
\nabla_{\bm x}\log p_t(\bm x)
=
s_t^{-2}\bigl(\bar{\bm m}(\bm x,t)-\bm x\bigr).
\]
Thus, in this mixture setting, the nonlinear part of the score is determined by a kernel regression over the component means~\cite{lyu2025resolving}.
\end{remark}

\subsection{Stability of the probability flow ODE}
\label{subsec:stability_of_the_probability_flow_ode}

We next study the stability of the deterministic transport map induced by the OU forward process. To this end, we work with the probability flow ODE~\cite{song2020denoising,song2021sde}, a deterministic dynamical system whose time-\(t\) marginal agrees with the SDE marginal \(p_t\). It can be derived by rewriting the diffusion term in the Fokker--Planck equation as a transport term through \(\Delta p_t=\nabla\cdot\bigl(p_t\nabla\log p_t\bigr)\), which yields a continuity equation with velocity field \(\bm v(\bm x,t)\).

For the OU process in~\eqref{eq:ou_sde}, the probability flow ODE is
\begin{equation}
\label{eq:pf_ode}
\dot{\bm x}_t=\bm v(\bm x_t,t),
\quad
\text{where }\bm v(\bm x,t)\coloneqq-\beta_t\bigl(\bm x+\nabla_{\bm x}\log p_t(\bm x)\bigr).
\end{equation}
When \(\nabla_{\bm x}\log p_t(\bm x)\) is the exact score of the SDE marginal, the solution of~\eqref{eq:pf_ode} induces the same family of marginals \(\{p_t\}_{t\in[0,T]}\) as~\eqref{eq:ou_sde}. We denote the associated flow map by \(\Phi_{s\leftarrow t}\), meaning that if \(\bm x_\cdot\) solves~\eqref{eq:pf_ode}, then \(\Phi_{s\leftarrow t}(\bm x_t)=\bm x_s\). The probability flow ODE~\eqref{eq:pf_ode} is also directly connected to deterministic diffusion samplers. In particular, denoising diffusion implicit models (DDIM)~\cite{song2020denoising} can be viewed as a discretization of~\eqref{eq:pf_ode} in the small-step limit. Therefore, the stability properties of \(\Phi_{s\leftarrow t}\) provide a continuous-time perspective on how deterministic samplers propagate perturbations.

From~\eqref{eq:pf_ode} and the closed-form Hessian in~\eqref{eq:score_hess_display}, the Jacobian of the velocity field with respect to \(\bm x\) is
\begin{equation}
\label{eq:dv_formula}
\nabla_{\bm x} \bm v(\bm x,t)
=
-\beta_t\bigl(\bm I_d+\nabla_{\bm x}^2\log p_t(\bm x)\bigr)
=
-\beta_t\biggl[\Bigl(1-\frac{1}{s_t^2}\Bigr)\bm I_d+\frac{1}{s_t^4}\Sigma_\gamma(\bm x,t)\biggr].
\end{equation}
To bound the spatial Lipschitz constant of \(\bm v(\bm x,t)\), define the maximal pairwise separation of the component means by
\begin{equation*}
\mathrm{diam}(t)\coloneqq\max_{i,j}\lVert \bm m_i(t)-\bm m_j(t)\rVert_2=m_t\cdot\mathrm{diam}(0).
\end{equation*}
Since \(\Sigma_\gamma(\bm x,t)\) is the covariance matrix of a discrete distribution supported on \(\{\bm m_k(t)\}_{k=1}^K\), whose diameter is \(\mathrm{diam}(t)\), the scalar Popoviciu inequality gives, for every unit vector \(\bm u\),
\begin{equation*}
\bm u^\top\Sigma_\gamma(\bm x,t)\bm u
=\operatorname{Var}_\gamma(\bm u^\top\bm m_k(t))
\leq \frac{\mathrm{diam}(t)^2}{4}.
\end{equation*}
Taking the supremum over \(\lVert\bm u\rVert_2=1\) yields
\begin{equation*}
\lVert \Sigma_\gamma(\bm x,t)\rVert_{2}\leq \frac{\mathrm{diam}(t)^2}{4}.
\end{equation*}
This gives the following spatial Lipschitz and flow stability bounds.

\begin{proposition}
\label{prop:flow_stability}
Define \(L(t)\) by
\begin{equation}
\label{eq:L_def}
L(t)\coloneqq\beta_t\biggl(\Bigl|1-\frac{1}{s_t^2}\Bigr|+\frac{\mathrm{diam}(t)^2}{4s_t^4}\biggr).
\end{equation}
Then for any \(\bm x\in\mathbb R^d\) and almost every \(t\in[0,T]\),
\begin{equation}
\label{eq:lip_v}
\lVert\nabla_{\bm x} \bm v(\bm x,t)\rVert_{2}\leq L(t).
\end{equation}
Moreover, for any \(0\leq s<t\leq T\) and any \(\bm x,\bm y\in\mathbb R^d\), the flow map \(\Phi_{s\leftarrow t}\) satisfies
\begin{equation}
\label{eq:flow_lip}
\lVert\Phi_{s\leftarrow t}(\bm x)-\Phi_{s\leftarrow t}(\bm y)\rVert_2
\le
\exp\Bigl(\int_s^t L(\tau)\mathrm d\tau\Bigr)\lVert\bm x-\bm y\rVert_2.
\end{equation}
\end{proposition}

\begin{proof}
By~\eqref{eq:dv_formula} and the triangle inequality for operator norms,
\begin{equation*}
\lVert\nabla_{\bm x} \bm v(\bm x,t)\rVert_{2}
\le
\beta_t\biggl(\Bigl|1-\frac{1}{s_t^2}\Bigr|+\frac{1}{s_t^4}\lVert\Sigma_\gamma(\bm x,t)\rVert_{2}\biggr)
\leq L(t),
\end{equation*}
which proves~\eqref{eq:lip_v}. For~\eqref{eq:flow_lip}, applying Gr\"onwall's inequality to the difference of two solutions of~\eqref{eq:pf_ode} gives the stated flow stability estimate.
\end{proof}

\begin{remark}
The bound \(L(t)\) in~\eqref{eq:L_def} contains two terms. The term \(\beta_t|1-1/s_t^2|\) is already present in the Gaussian case, while the term \(\beta_t\mathrm{diam}(t)^2/(4s_t^4)\) is caused by the separation of mixture components. The latter term vanishes for a single Gaussian component and grows when the component variance scale \(s_t^2\) is small or when the mixture means are widely separated. Thus, for strongly multimodal distributions at low noise, the velocity field can be highly sensitive to perturbations, leading to large flow amplification in~\eqref{eq:flow_lip}.
\end{remark}

\begin{remark}
Inequality~\eqref{eq:flow_lip} shows that approximating the full flow map over a long time horizon can accumulate a large amplification factor, thereby magnifying local approximation or numerical errors. This motivates using a time grid to partition the generation process and approximating the flow map on each subinterval separately, as in piecewise distillation methods. Such a strategy helps control error propagation across the composed transport.
\end{remark}

\section{Approximation Theory for Score Learning}
\label{sec:approximation_theory_for_score_learning}

This section studies the network complexity of approximating score functions. For each fixed time \(t\), the target map is \(\bm x\mapsto \nabla_{\bm x}\log p_t(\bm x)\), and we ask how large a neural network must be to approximate this map to a prescribed accuracy. We work with a ReLU--ReQU hypothesis class \(\mathcal F\), which is convenient for implementing the algebraic operations and normalized exponential weights appearing in the Gaussian-mixture score. The main result shows that, in the GMM--OU setting introduced in~\autoref{sec:theoretical_setup}, the score can be approximated uniformly over \(t\in[0,T]\) with explicit quadratic and cubic logarithmic factors in the target accuracy and the model scales for the depth and width, respectively.

\subsection{Hypothesis class}
\label{subsec:hypothesis_class_and_optimal_uniform_approximation_error}

We first define the network architecture used for score approximation. We use feedforward networks with mixed ReLU and ReQU activations. Approximation properties of such architectures have been studied extensively~\cite{belomestny2023simultaneous,li2019better,yang2023nearly,zhou2025expressive}. ReLU units allow exact implementations of truncation and maximum-type operations, while ReQU units provide exact constructions of basic polynomial operations such as squares and products. This mixed architecture is therefore well suited to the algebraic and normalization operations in the GMM score.

Define the activation functions
\begin{equation*}
\sigma_1(z)\coloneqq\mathrm{ReLU}(z)=\max\{z,0\}=(z)_+,
\quad
\sigma_2(z)\coloneqq\mathrm{ReQU}(z)=(z)_+^2.
\end{equation*}
For any bias vector \(\bm b=(b_1,\dots,b_r)^\top\in\mathbb R^r\) and input \(\bm y=(y_1,\dots,y_r)^\top\in\mathbb R^r\), define the coordinatewise shifted activation operator
\begin{equation*}
\bm\sigma_{\ell,\bm b}(\bm y)
\coloneqq
\bigl(\sigma_\ell(y_1+b_1),\dots,\sigma_\ell(y_r+b_r)\bigr)^\top
\colon
\mathbb R^r\to\mathbb R^r,
\quad
\ell\in\{1,2\}.
\end{equation*}
For depth \(L\in\mathbb N\), width sequence \(\{w_i\}_{i=0}^{L+1}\), weight matrices \(\bm W_i\in\mathbb R^{w_{i+1}\times w_i}\), bias vectors \(\bm b_i\in\mathbb R^{w_{i+1}}\), and activation indices \(\ell_i\in\{1,2\}\), define
\begin{equation}
\label{eq:relu_requ_net}
\bm f\colon\mathbb R^{w_0}\to\mathbb R^{w_{L+1}},
\quad
\bm x\mapsto \bm f(\bm x)
\coloneqq
\bm W_L\bm\sigma_{\ell_L,\bm b_L}\bm W_{L-1}
\bm\sigma_{\ell_{L-1},\bm b_{L-1}}
\cdots
\bm W_1\bm\sigma_{\ell_1,\bm b_1}\bm W_0\bm x.
\end{equation}

Throughout, for a matrix \(\bm W=(W_{ij})\), we use the entrywise maximum norm \(\lVert\bm W\rVert_\infty\coloneqq\max_{i,j}|W_{ij}|\); for vectors, \(\lVert\cdot\rVert_\infty\) is the usual maximum norm. Thus, the parameter bound below controls every individual weight and bias.

To control network complexity, for structural parameters \((Q,G,S,B)\in[0,\infty]^4\), define \(\mathcal F(Q,G,S,B)\) as the class of functions \(\bm f\) of the form~\eqref{eq:relu_requ_net} satisfying
\begin{equation}
\label{eq:hypothesis_space_relu_requ}
\begin{aligned}
\mathcal F(Q,G,S,B)
=
\bigl\{
\bm f\colon\mathbb R^{w_0}\to\mathbb R^{w_{L+1}}
\bigm\vert{}
&L\le Q,\quad \max\{w_1, w_2, \cdots, w_L\}\le G,\\
&\displaystyle
\sum_{i=0}^L\lVert\bm W_i\rVert_0+
\sum_{i=1}^L\lVert\bm b_i\rVert_0\le S,\\
&\sup_{0\le k\le L}\lVert\bm W_k\rVert_\infty
\vee
\sup_{1\le k\le L}\lVert\bm b_k\rVert_\infty
\le B
\bigr\}.
\end{aligned}
\end{equation}
The class is not required to be globally bounded in output. This is necessary because the exact GMM score grows linearly as \(\lVert\bm x\rVert\to\infty\). When an output envelope is needed on a compact set \(D\), we use the notation
\begin{equation*}
\mathcal F_D(Q,G,S,B,F)
\coloneqq
\{\bm f\in\mathcal F(Q,G,S,B):\ \lVert\bm f\rVert_{L^\infty(D)}\le F\}.
\end{equation*}
When there is no ambiguity, we write \(\mathcal F\coloneqq\mathcal F(Q,G,S,B)\).

\subsection{Main approximation theorem for the GMM score}
\label{subsec:main_approximation_theorem_for_score_matching}

In the isotropic GMM--OU setting, the score has the explicit form
\begin{equation*}
\nabla_{\bm x}\log p_t(\bm x)
=
-\frac{1}{s_t^2}
\Bigl(
\bm x-\sum_{k=1}^K\gamma_k(\bm x,t)\bm m_k(t)
\Bigr),
\quad
\bm m_k(t)=m_t\bm\mu_k.
\end{equation*}
Moreover, the posterior mixture weights \(\gamma_k(\bm x,t)\) are obtained by normalizing exponential quadratic weights. More precisely, define the unnormalized log-weights
\begin{equation}
\label{eq:logit_def}
a_k(\bm x,t)
\coloneqq
\log\pi_k
-
\frac{\lVert\bm x-\bm m_k(t)\rVert^2}{2s_t^2},\quad 
\gamma_k(\bm x,t)
=
\frac{\exp\bigl(a_k(\bm x,t)\bigr)}
{\sum_{j=1}^K\exp\bigl(a_j(\bm x,t)\bigr)}.
\end{equation}
Because all mixture components have the same covariance, the quadratic term \(-\lVert\bm x\rVert_2^2/(2s_t^2)\) is common to all logits and cancels under softmax normalization. Hence, the posterior weights can equivalently be written as
\begin{equation}
\label{eq:affine_logit_def}
\gamma_k(\bm x,t)
=
\frac{\exp(b_k(\bm x,t))}{\sum_{j=1}^K\exp(b_j(\bm x,t))},
\qquad
b_k(\bm x,t)
\coloneqq
\log\pi_k
+
\frac{m_t}{s_t^2}\bm\mu_k^\top\bm x
-
\frac{m_t^2}{2s_t^2}\lVert\bm\mu_k\rVert_2^2.
\end{equation}
This affine-logit representation removes an unnecessary quadratic dependence on \(\bm x\) from the approximation argument. The following theorem gives a constructive approximation of the resulting score.

\begin{restatable}{theorem}{scoreLpApprox}
\label{thm:score_Lp_approx}
Assume~\autoref{ass:suzuki_schedule} and~\autoref{ass:gmm_ou_init}. Let \(p_t\) be given by~\eqref{eq:gmm_pt}, and let \(p\in[1,\infty)\). For every \(0<\varepsilon<1/2\), there exists a ReLU--ReQU network \(\widehat{\bm s}_\varepsilon\colon\mathbb R^d\times[0,T]\to\mathbb R^d\) such that, uniformly in \(t\in[0,T]\),
\begin{equation}
\label{eq:score_Lp_target}
(
\mathbb E_{\bm X_t\sim p_t}
[
\lVert
\widehat{\bm s}_\varepsilon(\bm X_t,t)
-
\nabla_{\bm x}\log p_t(\bm X_t)
\rVert_2^p
]
)^{1/p}
\le C\varepsilon.
\end{equation}
Define
\begin{equation*}
N_{\varepsilon,\sigma}
\coloneqq
2+d+K+\varepsilon^{-1}+\sigma+\sigma^{-1}.
\end{equation*}
The constant \(C\) depends only on \(p,T\), the schedule constants in~\autoref{ass:suzuki_schedule}, \(\pi_{\min}^{-1}\), and \(\mu_{\max}\coloneqq\max_k\lVert\bm\mu_k\rVert_2\); it has no additional dependence on \(d,K,\varepsilon\), or \(\sigma\).

The network may be chosen in \(\mathcal F(Q,G,S,B)\) with
\begin{equation*}
\begin{aligned}
Q&\le C_1\Bigl(\log K+\bigl(1+\log N_{\varepsilon,\sigma}\bigr)^2\Bigr),\\
G&\le C_2K\Bigl(d+\bigl(1+\log N_{\varepsilon,\sigma}\bigr)^3\Bigr),\\
S&\le C_3\Bigl[
Kd+(d+K)\Bigl(\log K+\bigl(1+\log N_{\varepsilon,\sigma}\bigr)^4\Bigr)
\Bigr],\\
B&\le C_4N_{\varepsilon,\sigma}^{11+2/p}
\le C_4N_{\varepsilon,\sigma}^{13},
\end{aligned}
\end{equation*}
where \(C_1,\ldots,C_4\) have the same allowed dependencies as \(C\) and no additional dependence on \(d,K,\varepsilon\), or \(\sigma\). The exponent \(11+2/p\) is obtained by tracking the scalar coefficient and softmax modules; the coarser exponent \(13\) is uniform over all \(p\in[1,\infty)\).
\end{restatable}

\begin{proof}
We sketch the proof here and defer the complete argument to~\autoref{appsec:proof_of_approx}. By Lemma~B.1 in~\cite{oko2023diffusion}, the scalar OU coefficient \(m_t\) admits a one-dimensional ReLU approximation with depth and width bounded by \(C\bigl(1+\log A_{\delta,\sigma}\bigr)^2\) and sparsity bounded by \(C\bigl(1+\log A_{\delta,\sigma}\bigr)^3\), where \(A_{\delta,\sigma}=2+\delta^{-1}+\sigma+\sigma^{-1}\). Since \(s_t^2=1+(\sigma^2-1)m_t^2,\) an exact ReQU square followed by the reciprocal module in Lemma~F.7 in~\cite{oko2023diffusion} gives a uniform approximation of \(s_t^{-2}\). For the reciprocal module, the depth, width, and sparsity are bounded by \(C\bigl(1+\log A_{\delta,\sigma}\bigr)^2\), \(C\bigl(1+\log A_{\delta,\sigma}\bigr)^3\), and \(C\bigl(1+\log A_{\delta,\sigma}\bigr)^4\), respectively. Combining it with the square, affine, and clipping modules gives the explicit coefficient-network parameter bound \(B_\delta\le C(1+\delta^{-2}+\sigma^4+\sigma^{-4})\).

Using the affine logits in~\eqref{eq:affine_logit_def}, the induced logit error is bounded by \(C(1+s_-^{-2})(1+\mu_{\max})^2\delta(1+\lVert\bm x\rVert_2), s_-^2\coloneqq\min\{1,\sigma^2\}.\) We subtract the maximum logit and clip below at a logarithmic level \(M=O(\log N_{\varepsilon,\sigma})\). The exponential and reciprocal functions are approximated by one-dimensional ReLU networks using Lemmas~B.1 and F.7 in~\cite{oko2023diffusion}, while the remaining finite algebraic products are realized exactly by constant-depth ReQU modules. The estimate \(\lVert\operatorname{softmax}(\bm u)-\operatorname{softmax}(\bm v)\rVert_1 \le 2\lVert\bm u-\bm v\rVert_\infty\) then transfers the logit error to the posterior weights. Choosing the posterior approximation tolerance proportional to \(s_-^2\varepsilon/(1+\mu_{\max})\) compensates for the factor \(s_t^{-2}\) in the score formula. The resulting scalar accuracies satisfy \(\delta^{-1}\le C N_{\varepsilon,\sigma}^{11/2+1/p}\) and \(\xi^{-1}\le C N_{\varepsilon,\sigma}^{3}\). Tracking the reciprocal and softmax parameter magnitudes then gives \(B\le C N_{\varepsilon,\sigma}^{11+2/p}\), while the displayed quadratic, cubic, and quartic logarithmic powers give the structural bounds. In the sparsity count, the term \(Kd\) records the affine logits and the final posterior mean, while~\autoref{lem:parallel_identity_padding} controls exact identity propagation of the \(d\) input coordinates and padding of the \(K\) intermediate channels across subnetworks of unequal depth. The construction is not globally bounded in output, consistently with the linear growth of the exact score.
\end{proof}

\begin{remark}
\label{rem:sigma_poly_dependence}
The displayed bounds distinguish structural complexity from parameter magnitude. The accuracy-dependent logarithmic powers are explicit: quadratic for depth, cubic for width, and quartic for sparsity. Since \(N_{\varepsilon,\sigma}\ge 2+\sigma+\sigma^{-1}\), the dependence of these structural quantities on \(\sigma^{-1}\) is controlled by the corresponding powers of \(1+\log N_{\varepsilon,\sigma}\). The refined sparsity estimate additionally records identity propagation of the \(d\) input coordinates and padding of the \(K\) intermediate channels. The parameter magnitude satisfies the explicit estimate \(B\le C N_{\varepsilon,\sigma}^{11+2/p} \le C N_{\varepsilon,\sigma}^{13}.\) Consequently, \(\log B \le \log C+\Bigl(11+\frac{2}{p}\Bigr)\log N_{\varepsilon,\sigma}.\) Thus the quantity entering the logarithm of the local Lipschitz budget still has only logarithmic dependence on \(\sigma^{-1}\), although \(B\) itself has an explicit polynomial dependence. In flow approximation, by contrast, the same low-noise scale enters the Gr\"onwall exponent and may be amplified to \(\exp\bigl(\frac{D_0^2}{16\sigma^2}\bigr)\) for \(D_0>0\) and \(\sigma \leq 1\).
\end{remark}

\begin{remark}
\label{rem:fixed_requ_layers}
In~\autoref{thm:score_Lp_approx}, ReLU--ReQU networks are used for constructive convenience. All accuracy-dependent scalar approximations, including those of \(m_t\), the exponential function, and the reciprocal function, are realized by ReLU subnetworks. ReQU units are used only for a fixed number of exact algebraic operations, such as squaring and multiplication. Consequently, the number of ReQU activation layers in the constructed network is bounded by an absolute constant independent of the target accuracy \(\varepsilon\); only their width may grow when several products are computed in parallel. The use of ReQU is not essential. A ReLU-only construction can be obtained by replacing the exact square and multiplication modules with standard ReLU approximants, while retaining dependence through fixed powers of \(1+\log(\varepsilon^{-1})\), up to changes in the constants and in the precise powers. We use the mixed ReLU--ReQU architecture to keep the algebraic part of the construction exact and the proof focused on the score approximation mechanism.
\end{remark}

\section{Flow Approximation and Stability-Balanced Distillation}
\label{sec:distillation_and_stability_guided_time_discretization}

In~\autoref{sec:approximation_theory_for_score_learning}, we established an approximation result for score learning in the GMM--OU setting. For any fixed \(t\), within a suitable ReLU--ReQU hypothesis class, one can achieve high-accuracy approximation in \(L^p(p_t)\) with accuracy-dependent depth, width, and sparsity factors of orders \(\bigl(1+\log(\varepsilon^{-1})\bigr)^2\), \(\bigl(1+\log(\varepsilon^{-1})\bigr)^3\), and \(\bigl(1+\log(\varepsilon^{-1})\bigr)^4\), respectively, apart from the displayed dependence on the model parameters. Consequently, the main bottleneck for distillation is often not whether one can fit \(\bm v(\cdot,t)\) accurately enough at each time, but rather how this error is amplified by the flow map induced by long-time integration. In this section, we quantify such amplification through a computable bound on the time-dependent Jacobian of the probability-flow ODE, use the resulting stability factor to prove an efficient approximation result for finite-time flow maps by deep residual compositions, and then derive practical principles for choosing a time grid.

\subsection{Flow Stability from Jacobian Bounds}
\label{subsec:flow_stability_profile_and_its_interpretation}

Recall from~\autoref{prop:flow_stability} that the flow map of the probability flow ODE \(\dot{\bm x}_t=\bm v(\bm x_t,t)\) satisfies
\begin{equation*}
\lVert \Phi_{s\leftarrow t}(\bm x)-\Phi_{s\leftarrow t}(\bm y)\rVert_2
\leq \exp\Bigl(\int_s^t L(u)\mathrm du\Bigr)\lVert \bm x-\bm y\rVert_2,
\end{equation*}
where
\begin{equation}
\label{eq:L_profile_distill}
L(t)=\beta_t\Bigl(\Bigl|1-\frac{1}{s_t^2}\Bigr|+\dfrac{1}{4}\mathrm{diam}(t)^2 s_t^{-4}\Bigr).
\end{equation}
Here
\[
D_0\coloneqq\max_{i,j}\lVert\bm\mu_i-\bm\mu_j\rVert_2,
\qquad
s_t^2=m_t^2\sigma^2+\sigma_t^2,
\qquad
\mathrm{diam}(t)=m_tD_0.
\]

For clarity, we assume a constant schedule \(\beta_t\equiv\beta>0\). Then
\begin{equation}
\label{eq:constant_beta}
m_t=\exp(-\beta t),\quad
\sigma_t^2=1-\exp(-2\beta t),\quad
s_t^2=1-(1-\sigma^2)\exp(-2\beta t).
\end{equation}
In this case, \(L(t)\) has two contributions: the Gaussian variance-mismatch term \(\beta|1-s_t^{-2}|\) and the multimodality term \(\beta\mathrm{diam}(t)^2s_t^{-4}/4\). The following regime description is intended for the low-noise case \(0<\sigma<1\), in which \(s_t^2\) increases along the forward OU path. Their relative sizes lead to the following useful regimes:

\begin{itemize}
\item \textbf{Noise-dominated regime.} Near the high-noise end \(t\approx T\), \(s_t^2\) is close to one and the contracted component means are less separated. Both contributions to \(L(t)\) can then be small, so the velocity field is comparatively smooth.
\item \textbf{Variance-shrinkage regime.} As one moves toward smaller \(t\), the variance scale \(s_t^2\) decreases. In the regime where \(s_t^{-2}\) is larger than the baseline scale but the multimodal term is not yet dominant, one has \(L(t)\approx\beta s_t^{-2}\). Fix \(c_0>1\) and define the crossover time \(t_{\mathrm{sg}}\) by \(s_{t_{\mathrm{sg}}}^{-2}=c_0\). If \(\sigma^2<1/c_0\), this crossover time is given by
\begin{equation*}
t_{\mathrm{sg}}
=
\frac{1}{2\beta}
\log\Bigl(\frac{1-\sigma^2}{1-1/c_0}\Bigr).
\end{equation*}
If \(\sigma^2\geq 1/c_0\), the threshold is not reached. Even when \(\sigma^2<1/c_0\), it lies outside the observed horizon if \(t_{\mathrm{sg}}>T\).
\item \textbf{Mixture-dominated regime.} Closer to the low-noise end, the component means can be well separated relative to the variance scale, and the multimodal term may dominate:
\begin{equation*}
L(t)\approx \frac{\beta}{4}\mathrm{diam}(t)^2s_t^{-4}.
\end{equation*}
We define \(t_{\mathrm{mix}}\) by the balance condition \(\frac{1}{4}\mathrm{diam}(t_{\mathrm{mix}})^2s_{t_{\mathrm{mix}}}^{-4}=s_{t_{\mathrm{mix}}}^{-2}\) between the multimodal term and the variance-shrinkage term. Using~\eqref{eq:constant_beta}, this gives
\begin{equation*}
t_{\mathrm{mix}}
=
\frac{1}{2\beta}
\log\Bigl(\frac{D_0^2+4(1-\sigma^2)}{4}\Bigr).
\end{equation*}
This crossover is relevant only when the logarithm is positive and \(t_{\mathrm{mix}}\in[0,T]\). In particular, the positivity condition is \(D_0^2>4\sigma^2\). When \(D_0\) is large and the low-noise marginal remains strongly multimodal, \(L(t)\) can become large, making long-interval distillation difficult.
\end{itemize}

\subsection{A priori stability-certificate comparison and its engineering relevance}
\label{subsec:a_priori_lipschitz_mismatch_and_engineering_relevance}

The previous subsection characterized the stability bound \(L(t)\) and identified its dominant regimes. We now turn this bound into an a priori stability-budget check for one-step distillation. Over a fixed low-noise tail window of length \(\Delta\), the standard stability estimate propagates local perturbations by the factor \(\exp(\int_0^\Delta L(t)\mathrm dt)\). When \(\sigma\) is small and \(D_0\) is large, this factor has exponential dependence on \(D_0^2/\sigma^2\). By contrast, the Lipschitz scale that can be certified for a one-step student is constrained by its architecture and by the range on which the student is evaluated. When the numerical value of the tail Gr\"onwall certificate exceeds this certified local Lipschitz budget, the two a priori scales are mismatched; this comparison is diagnostic and does not by itself imply an approximation lower bound.

For simplicity, we continue to assume \(\beta_t\equiv\beta>0\) and define the fixed tail window
\begin{equation}
\label{eq:delta_and_tail}
\Delta \coloneqq \frac{\log 2}{2\beta},
\qquad
\Lambda_{\mathrm{tail}}\coloneqq \int_0^\Delta L(t)\mathrm dt.
\end{equation}
For \(R>0\), let \(B_R\coloneqq\{\bm x\in\mathbb R^d:\lVert\bm x\rVert_2\le R\}\) and define the local Lipschitz seminorm
\begin{equation*}
\mathrm{Lip}(\bm f;B_R)
\coloneqq
\sup_{\substack{\bm x\ne\bm y\\ \bm x,\bm y\in B_R}}
\frac{\lVert\bm f(\bm x)-\bm f(\bm y)\rVert_2}{\lVert\bm x-\bm y\rVert_2}.
\end{equation*}
A ReLU activation is globally \(1\)-Lipschitz, whereas the derivative of the ReQU activation \(\sigma_2(z)=(z_+)^2\) grows linearly with its input. As noted in~\autoref{rem:fixed_requ_layers}, the construction in~\autoref{thm:score_Lp_approx} realizes all accuracy-dependent scalar approximations by ReLU subnetworks and uses ReQU only for a fixed number of exact squaring and multiplication operations. Consequently, there exists an absolute integer \(r_{\mathrm{ReQU}}\ge1\), independent of \(\varepsilon\), \(d\), and \(K\), such that the constructed score networks contain at most \(r_{\mathrm{ReQU}}\) ReQU activation layers.

For \(r_{\mathrm{ReQU}}\in\mathbb N\), let
\begin{equation*}
\begin{aligned}
\mathcal F_{d\to d}^{(r_{\mathrm{ReQU}})}(Q,G,S,B)
\coloneqq
\bigl\{
\bm f\in\mathcal F(Q,G,S,B)
\bigm|{}
&\bm f\colon\mathbb R^d\to\mathbb R^d,\\
&\#\{\text{ReQU activation layers of }\bm f\}\le r_{\mathrm{ReQU}}
\bigr\}.
\end{aligned}
\end{equation*}
All remaining activation layers are required to be ReLU. After fixing the time input and absorbing a model-dependent constant enlargement of the parameter bound, the score networks constructed in~\autoref{thm:score_Lp_approx} belong to such a class for an absolute \(r_{\mathrm{ReQU}}\). Define the corresponding local Lipschitz budget by
\begin{equation}
\label{eq:lipschitz_budget}
\mathrm{LipBudget}^{(r_{\mathrm{ReQU}})}_R
\coloneqq
\sup_{\bm f\in\mathcal F_{d\to d}^{(r_{\mathrm{ReQU}})}(Q,G,S,B)}
\mathrm{Lip}(\bm f;B_R).
\end{equation}
Because the input set is bounded and the depth, width, weights, and biases are uniformly bounded, the inputs to the finitely many ReQU layers are automatically bounded on \(B_R\). Thus, no additional preactivation-range assumption is required. With these definitions,~\autoref{thm:lipschitz_mismatch_distill} gives a sufficient condition under which the tail stability scale exceeds this certified local student budget.

\begin{restatable}{theorem}{LipschitzMismatchDistill}
\label{thm:lipschitz_mismatch_distill}
Assume~\eqref{eq:constant_beta}, \(T\ge\Delta=(\log 2)/(2\beta)\), and use the definitions in~\eqref{eq:delta_and_tail} and~\eqref{eq:lipschitz_budget}. Let
\begin{equation*}
\overline G\coloneqq\max\{d,G\},
\qquad
\Gamma_R\coloneqq\max\bigl\{e,C\bigl(1+B\overline G\bigr)(1+R)\bigr\},
\end{equation*}
where \(C\ge1\) is a universal constant. Then there exists a constant \(c_{r_{\mathrm{ReQU}}}\ge1\), depending only on \(r_{\mathrm{ReQU}}\), such that, for every \(\sigma>0\),
\begin{equation}
\label{eq:tail_exact_certificate_lower}
\Lambda_{\mathrm{tail}}
\ge
\frac{D_0^2}{8\sigma^2(1+\sigma^2)}.
\end{equation}
In particular, if \(\sigma\le1\), then
\begin{equation*}
\Lambda_{\mathrm{tail}}\ge\frac{D_0^2}{16\sigma^2},
\qquad
\exp(\Lambda_{\mathrm{tail}})
\ge
\exp\biggl(\frac{D_0^2}{16\sigma^2}\biggr).
\end{equation*}
Moreover,
\begin{equation}
\label{eq:student_local_budget_corrected}
\mathrm{LipBudget}^{(r_{\mathrm{ReQU}})}_R
\le
\Gamma_R^{c_{r_{\mathrm{ReQU}}}(Q+1)}.
\end{equation}
Consequently, if \(\sigma\le1\) and
\begin{equation}
\label{eq:certificate_mismatch_condition}
\frac{D_0^2}{16\sigma^2}
>
c_{r_{\mathrm{ReQU}}}(Q+1)\log\Gamma_R,
\end{equation}
then
\begin{equation*}
\exp(\Lambda_{\mathrm{tail}})
>
\Gamma_R^{c_{r_{\mathrm{ReQU}}}(Q+1)}
\ge
\mathrm{LipBudget}^{(r_{\mathrm{ReQU}})}_R.
\end{equation*}
This is a comparison between a computed Gr\"onwall certificate and a certified local architectural budget. It does not assert a lower bound on the actual Lipschitz constant of the teacher flow or on the \(L^p\)-approximation error of the student class. In particular, no ordering between \(\operatorname{Lip}(\Phi_{0\leftarrow\Delta};B_R)\) and \(\mathrm{LipBudget}^{(r_{\mathrm{ReQU}})}_R\) follows from this certificate comparison alone.
\end{restatable}

The proof is deferred to~\autoref{app:proof_lipschitz_mismatch_distill}. We now compare the tail stability scale in~\autoref{thm:lipschitz_mismatch_distill} with the architectural budgets implied by~\autoref{thm:score_Lp_approx}. By~\autoref{rem:fixed_requ_layers}, the construction contains at most \(r_{\mathrm{ReQU}}\) ReQU activation layers, where \(r_{\mathrm{ReQU}}\) is an absolute constant independent of \(\varepsilon,d,K\), and \(\sigma\). Recall that \(N_{\varepsilon,\sigma} \coloneqq 2+d+K+\varepsilon^{-1}+\sigma+\sigma^{-1}.\)~\autoref{thm:score_Lp_approx} gives
\begin{equation*}
Q
\lesssim
\log K+\bigl(1+\log N_{\varepsilon,\sigma}\bigr)^2,
\quad
G
\lesssim
K\Bigl(d+\bigl(1+\log N_{\varepsilon,\sigma}\bigr)^3\Bigr),
\quad
B
\lesssim
N_{\varepsilon,\sigma}^{11+2/p}.
\end{equation*}
Since \(K,d\le N_{\varepsilon,\sigma}\) and \(\bigl(1+\log N_{\varepsilon,\sigma}\bigr)^3\le C N_{\varepsilon,\sigma}\), the width bound implies \(\overline G\le C N_{\varepsilon,\sigma}^2\). Hence
\[
B\overline G
\le
C N_{\varepsilon,\sigma}^{13+2/p}.
\]
Because \(\Gamma_R = \max\bigl\{e,C\bigl(1+B\overline G\bigr)(1+R)\bigr\},\) it follows that
\[
\log\Gamma_R
\le
C+\log(1+R)+\biggl(13+\frac{2}{p}\biggr)\log N_{\varepsilon,\sigma}.
\]
Substituting into~\eqref{eq:student_local_budget_corrected} gives
\begin{equation*}
c_{r_{\mathrm{ReQU}}}(Q+1)\log\Gamma_R
=
\mathcal O\Bigl(
\bigl(\log K+\bigl(1+\log N_{\varepsilon,\sigma}\bigr)^2\bigr)
\bigl(1+\log(1+R)+\bigl(13+2/p\bigr)\log N_{\varepsilon,\sigma}\bigr)
\Bigr).
\end{equation*}
Since \(r_{\mathrm{ReQU}}\) is an absolute constant, the finitely many ReQU layers change only the constant in this exponent. For fixed \(d,K,R\), and \(\varepsilon\), the certified architectural exponent therefore grows at most cubically in \(1+\log(\sigma^{-1})\). More generally, for the localization radius
\[
R=R_\varepsilon\asymp\sqrt{1+\log(\varepsilon^{-1})}
\]
used in the proof of~\autoref{thm:resnet_flow_tail_approx}, one has \(\log(1+R_\varepsilon) = \mathcal O\bigl(\log\log(\varepsilon^{-1})\bigr),\) which is absorbed into the displayed cubic logarithmic envelope. In contrast, for \(\sigma\le1\),
\[
\Lambda_{\mathrm{tail}}
\ge
\frac{D_0^2}{16\sigma^2}.
\]
Consequently, provided that \(D_0>0\) is fixed, the lower certificate grows at least at the rate \(\sigma^{-2}\), whereas the logarithm of the certified architectural budget grows at most cubically in \(1+\log(\sigma^{-1})\). Hence the certificate mismatch holds for all sufficiently small \(\sigma\). More generally, the same conclusion holds along any joint limit for which
\[
c_{r_{\mathrm{ReQU}}}(Q+1)\log\Gamma_R
=o\biggl(\frac{D_0^2}{\sigma^2}\biggr).
\]

This result should be read as an a priori stability-budget diagnostic rather than as an \(L^p\)-approximation lower bound. A large Gr\"onwall upper scale does not by itself prove that every low-Lipschitz student has large distributional error. It does show that, in the low-noise well-separated regime, the available upper certificates do not certify that the teacher flow lies within the local Lipschitz range of the prescribed student class. In practice, one often reuses the score-network architecture as the student model, so \(Q\) and \(G\) are essentially fixed, and the weight scale \(B\) cannot be increased arbitrarily without destabilizing training and inference. Consequently, whenever the explicit condition~\eqref{eq:certificate_mismatch_condition} holds, long-interval one-step distillation is unfavorable from the viewpoint of this stability-budget diagnostic.

Motivated by this mismatch, we next show that deep residual compositions can approximate the flow map, with the number of residual blocks \(n\) serving as the key degree of freedom. Composing short-interval maps distributes the representation task across local updates, so the entire long-horizon transport need not be represented by one block. The present worst-case end-to-end estimate nevertheless retains the cumulative factor \(\exp(\int L)\). We first define the residual architecture class.

\begin{definition}
\label{def:stab_adaptive_resnet}
Let \(\mathcal F\coloneqq\mathcal F(Q,G,S,B)\) be the ReLU--ReQU hypothesis class defined in~\eqref{eq:hypothesis_space_relu_requ}. For \(R\in(0,\infty]\), let \(\mathrm{Clip}_R\colon\mathbb R^d\to\mathbb R^d\) be the coordinatewise clipping map onto \([-R,R]^d\), with \(\mathrm{Clip}_\infty\) equal to the identity. For any integer \(n\geq1\), define the clipped residual network class \(\mathcal R_{n,R}(\mathcal F)\) as the set of maps
\begin{equation*}
\Psi\colon\mathbb R^d\to\mathbb R^d,
\quad
\Psi\coloneqq \Psi_n\circ\cdots\circ\Psi_1,
\end{equation*}
where each residual block has the form
\begin{equation}
\label{eq:res_block_general}
\Psi_k(\bm x)
\coloneqq
\mathrm{Clip}_R\bigl(\bm x+\alpha_k\widehat{\bm v}_k(\bm x)\bigr),
\quad
k=1,\dots,n,
\end{equation}
with \(\alpha_k\in\mathbb R\) a scalar step size and \(\widehat{\bm v}_k\) realized by a vector-valued network in \(\mathcal F\). When \(R=\infty\), this reduces to the unclipped residual class and we write \(\mathcal R_n(\mathcal F)\coloneqq\mathcal R_{n,\infty}(\mathcal F)\). In the probability-flow setting~\eqref{eq:pf_ode} with the constant schedule \(\beta_t\equiv\beta\), we often parameterize the velocity approximation through a score network:
\begin{equation}
\label{eq:vhat_from_score}
\widehat{\bm v}_k(\bm x)
\coloneqq
-\beta\bigl(\bm x+\widehat{\bm s}_k(\bm x)\bigr),
\quad
\widehat{\bm s}_k\in\mathcal F.
\end{equation}
Equivalently, one may use a time-conditioned score network \(\widehat{\bm s}(\cdot,t)\in\mathcal F\) and set
\begin{equation*}
\widehat{\bm v}_k(\bm x)
=
-\beta\bigl(\bm x+\widehat{\bm s}(\bm x,\tau_k)\bigr),
\end{equation*}
where \(\tau_k\) is the prescribed time condition in the \(k\)-th block. The affine identity branch and multiplication by \(-\beta\) in~\eqref{eq:vhat_from_score} can be absorbed by a constant enlargement of \((Q,G,S,B)\); all complexity statements below use this convention. The clipping map is exactly representable by ReLU coordinatewise and is used only for localization in the proof.
\end{definition}

\begin{remark}
The main modeling choice in~\autoref{def:stab_adaptive_resnet} is the function class used inside each residual block. In standard ResNets~\cite{he2016deep}, the residual branch is often shallow. Here we allow the blockwise velocity approximator to belong to the ReLU--ReQU class \(\mathcal F\). This analysis-driven choice lets us reuse the score approximation theorem without separately analyzing how shallow networks implement multiplication, exponentials, and normalization. The choice is not essential: similar conclusions should extend to shallow residual branches, but the proof would require approximation tools tailored to shallow networks, such as the Barron space framework~\cite{barron2002universal} or approximation theory for shallow sigmoid networks~\cite[Theorem 2.1]{mhaskar1996neural}. We do not pursue this extension here.
\end{remark}

\begin{restatable}{theorem}{ResnetFlow}
\label{thm:resnet_flow_tail_approx}
Assume~\autoref{ass:gmm_ou_init} and consider the probability-flow ODE~\eqref{eq:pf_ode} with constant schedule \(\beta_t\equiv\beta>0\). Let \(\Phi_{0\leftarrow T}\colon\mathbb R^d\to\mathbb R^d\) be the exact flow map, and set \(\Lambda_T\coloneqq\int_0^T L(t)\mathrm dt\). Fix \(p\in[1,\infty)\) and \(0<\varepsilon<1/2\).

There exist constants \(C_0,\ldots,C_5>0\), independent of \(\varepsilon\), an integer
\begin{equation}
\label{eq:block_count_corrected}
n
\le
1+C_0T\varepsilon^{-1}\sqrt{1+\log(\varepsilon^{-1})},
\end{equation}
a localization radius \(R_\varepsilon\asymp\sqrt{1+\log(\varepsilon^{-1})}\), and score networks \(\{\widehat{\bm s}_k\}_{k=1}^n\subset\mathcal F(Q,G,S,B)\), such that the clipped residual composition
\begin{equation*}
\Psi_{0\leftarrow T}
\coloneqq
\Psi_n\circ\cdots\circ\Psi_1
\in
\mathcal R_{n,R_\varepsilon+1}(\mathcal F)
\end{equation*}
satisfies
\begin{equation}
\label{eq:resnet_flow_error_corrected}
(
\mathbb E_{\bm X_T\sim p_T}
[
\lVert
\Psi_{0\leftarrow T}(\bm X_T)
-
\Phi_{0\leftarrow T}(\bm X_T)
\rVert_2^p
]
)^{1/p}
\le
C_1e^{\Lambda_T}\varepsilon.
\end{equation}
The constants depend only on \(p,d,\beta,T,\sigma,\pi_{\min}\) and the mixture means and weights.

Each residual block uses a compact fixed-time score approximant. Define
\[
N_{\varepsilon,\sigma}
\coloneqq
2+d+K+\varepsilon^{-1}+\sigma+\sigma^{-1}.
\]
One may choose
\begin{equation*}
\begin{aligned}
Q&\le C_2\Bigl(\log K+\bigl(1+\log N_{\varepsilon,\sigma}\bigr)^2\Bigr),\\
G&\le C_3K\Bigl(d+\bigl(1+\log N_{\varepsilon,\sigma}\bigr)^3\Bigr),\\
S&\le C_4\Bigl[
Kd+(d+K)\Bigl(\log K+\bigl(1+\log N_{\varepsilon,\sigma}\bigr)^4\Bigr)
\Bigr],\qquad
B\le C_5N_{\varepsilon,\sigma}^{10}.
\end{aligned}
\end{equation*}
Thus, for every fixed horizon \(T\), the block count is \(\widetilde{\mathcal O}_{T,\mathrm{model}}(\varepsilon^{-1})\), while the per-block complexity remains at the compact fixed-time score-approximation scale.
\end{restatable}

\begin{proof}
The full proof is in~\autoref{appsec:proof_resnet_flow_tail_approx}. Localize the exact trajectory to
\[
\mathcal G_{R_\varepsilon}
=
\Biggl\{
\sup_{t\in[0,T]}\lVert\bm X_t^\star\rVert_2\le R_\varepsilon
\Biggr\},
\]
whose complement has Gaussian tail probability when \(R_\varepsilon\asymp\sqrt{\log(\varepsilon^{-1})}\). On the corresponding compact set, the affine-logit construction gives fixed-time score networks with uniform error \(\eta\).

For a step of length \(h\), global linear growth of \(\bm v\) and \(\partial_t\bm v\) yields the first-order consistency estimate
\begin{equation*}
\bigl\lVert
\Phi_{t-h\leftarrow t}(\bm x)-\bigl(\bm x-h\bm v(\bm x,t)\bigr)
\bigr\rVert_2
\le C(1+\lVert\bm x\rVert_2)h^2.
\end{equation*}
Adding and subtracting the exact segment map at the numerical iterate gives a recursion with segment multiplier
\(
\exp(\int_{t_k}^{t_{k-1}}L(u)\mathrm du)
\)
and local defect \(h\beta\eta+C(1+R_\varepsilon)h^2\). Taking \(\eta=\varepsilon\) and \(h\lesssim\varepsilon/(1+R_\varepsilon)\) gives~\eqref{eq:block_count_corrected} and an \(O(e^{\Lambda_T}\varepsilon)\) error on the localized event. Projection nonexpansiveness preserves this recursion even if coordinatewise clipping is active. On the complement, clipping bounds the numerical endpoint; H\"older's inequality and the Gaussian tail estimate complete~\eqref{eq:resnet_flow_error_corrected}.
\end{proof}

\begin{remark}
The residual class \(\mathcal R_{n,R}(\mathcal F)\) in~\autoref{def:stab_adaptive_resnet} is stated in a block-wise form: each residual block uses a time-independent score network \(\widehat{\bm s}_k\in\mathcal F\). This formulation is convenient for the approximation proof, since each short-time transport map can be treated separately. In practice, one may instead use a single time-conditioned score network \(\widehat{\bm s}(\bm x,t)\) and set \(\widehat{\bm s}_k(\bm x)=\widehat{\bm s}(\bm x,\tau_k)\) at prescribed time conditions \(\tau_k\). This gives a parameter-sharing implementation of the same multistep residual structure. The theorem is stated in the block-wise form because it avoids imposing additional coupling across different time blocks; extending the approximation bound to a shared time-conditioned parametrization only requires approximating these block-wise networks jointly at the selected time values.
\end{remark}

\begin{remark}
\autoref{thm:resnet_flow_tail_approx} is an approximation result. It asserts the existence of parameters that achieve the stated error, but it does not claim that a specific training algorithm finds them. A possible route for connecting approximation to optimization is provided by mean-field analyses of overparameterized networks, which model training as a measure-valued gradient flow in the infinite-width limit and prove convergence under suitable assumptions~\cite{ding2022overparameterization,mei2019mean,sirignano2020mean}. In a related direction, the work in~\cite{gao2026toward} analyzes distillation through operator merging and explicitly targets the optimization error induced by distillation in the Gaussian setting. Establishing such a learnability guarantee for the present setting, and relating it quantitatively to the stability factor \(\exp(\Lambda_T)\), is outside the scope of this work.
\end{remark}

\begin{remark}
A Transformer layer combines residual updates with softmax normalization \cite{vaswani2017attention}. For example, self-attention has the form
\begin{equation*}
\mathrm{Attn}(\bm Q,\bm K,\bm V)
=
\mathrm{softmax}\bigl(\bm Q\bm K^\top/\sqrt d\bigr)\bm V,
\end{equation*}
and a Transformer block updates \(\bm x\) through a residual map \(\bm x\mapsto \bm x+\mathrm{Block}(\bm x)\). These two structural features parallel the ingredients used in~\autoref{thm:resnet_flow_tail_approx}. It is therefore natural to ask whether Transformer classes admit approximation guarantees analogous to~\autoref{thm:resnet_flow_tail_approx}, possibly with parameter sharing across layers. We include Transformers in the experiments as a complementary architecture family.
\end{remark}

\autoref{thm:resnet_flow_tail_approx} explains why multistep residual distillation is a natural approximation class for the probability-flow map. For any increasing partition \(0=\tau_0<\cdots<\tau_n=T\), the exact flow satisfies the semigroup decomposition
\begin{equation*}
\Phi_{0\leftarrow T}
=
\Phi_{\tau_0\leftarrow \tau_1}
\circ\cdots\circ
\Phi_{\tau_{n-1}\leftarrow \tau_n}.
\end{equation*}
The residual architecture builds the same structure into the hypothesis space, with each block approximating a short-time transport map. This provides a theoretical justification for multistep distillation and is consistent with the empirical observation that segment-wise training is often more stable and easier to control~\cite{ding2024characteristic}. Thus, depth is not only a way to enlarge a generic function class; here it distributes the representation of the long-horizon dynamics over shorter intervals with skip connections. In the GMM--OU setting, the difficulty is governed by both fixed-time score approximation and dynamical amplification. In the low-noise, well-separated regime, \(\Lambda_T=\int_0^T L(t)\mathrm dt\) can be large, and the present global estimate still contains \(\exp(\Lambda_T)\). The construction therefore avoids representing the full transport within a single block, but it does not remove the cumulative Gr\"onwall factor from the proved end-to-end bound. The time-dependent effect of local errors motivates the stability-balanced discretization rule in~\autoref{subsec:stability_balanced_time_discretization}.

\subsection{Stability-balanced time discretization}
\label{subsec:stability_balanced_time_discretization}

We now use the computable stability certificate to guide time discretization. We refine intervals where \(L(t)\) is large so that the segment certificate \(\Lambda_k\) is balanced. A small \(\Lambda_k\) gives a smaller Gr\"onwall multiplier in the error recursion, whereas a large \(\Lambda_k\) assigns a larger worst-case amplification to that segment. This is a certificate-based design principle; it does not identify the exact Lipschitz constant of each teacher segment.

Consider a decreasing generation grid
\begin{equation*}
T=t_0>t_1>\cdots>t_n=0,
\end{equation*}
where the \(k\)-th segment is \([t_k,t_{k-1}]\). Define the segmentwise stability exponent
\begin{equation*}
\Lambda_k \coloneqq \int_{t_k}^{t_{k-1}} L(u)\mathrm du,
\qquad
A(t)\coloneqq\int_0^t L(u)\mathrm du,
\end{equation*}
and denote the corresponding amplification factor by \(\exp(\Lambda_k)\). By additivity,
\begin{equation*}
\sum_{k=1}^n\Lambda_k
=
\int_0^T L(u)\mathrm du
=
A(T).
\end{equation*}
Hence every partition with \(n\) segments satisfies
\begin{equation*}
\max_{1\le k\le n}\Lambda_k
\ge
\frac1n\sum_{k=1}^n\Lambda_k
=
\frac{A(T)}{n}.
\end{equation*}
Equality holds if and only if all segmentwise exponents are equal. Therefore the cumulative-stability partition solves the certificate-level minimax problem
\begin{equation*}
\min_{T=t_0>\cdots>t_n=0}
\max_{1\le k\le n}\Lambda_k
=
\frac{A(T)}{n},
\end{equation*}
and equivalently minimizes the largest segmentwise amplification factor \(\max_k e^{\Lambda_k}\). The balancing condition is thus
\begin{equation}
\label{eq:equalize_Lambda_distill}
\Lambda_k
=
\frac{1}{n}\int_0^T L(u)\mathrm du,
\quad k=1,\dots,n.
\end{equation}
For the decreasing grid above, \eqref{eq:equalize_Lambda_distill} is equivalent to
\begin{equation}
\label{eq:A_uniform_partition}
A(t_k)=\frac{n-k}{n}A(T),\quad k=0,1,\dots,n.
\end{equation}
When \(A(T)>0\), define the generalized inverse \(A^{\leftarrow}(a)\coloneqq\inf\{t\in[0,T]:A(t)\ge a\}\). The increasing grid may be written as \(\tau_j=A^{\leftarrow}(jA(T)/n)\), with endpoints set exactly to \(\tau_0=0\) and \(\tau_n=T\), and the generation grid is \(t_k=\tau_{n-k}\). If \(A(T)=0\), every segment has zero certificate and we use the uniform grid. This convention also handles flat portions of the numerical cumulative profile without dividing by zero. The minimax property concerns the largest segmentwise stability certificate only; it does not by itself minimize the full end-to-end approximation bound unless the segmentwise local defects are also controlled at comparable scales.

In the GMM--OU setting, \(L(t)\) is given by \eqref{eq:L_profile_distill} together with the explicit formulas for \(m_t\), \(s_t\), and \(\mathrm{diam}(t)\), so \(A(t)\) can be computed by one-dimensional numerical integration. One can then solve a one-dimensional inverse problem to obtain the grid points. The noise-dominated regime is coarsely partitioned, the variance-shrinkage regime is moderately refined, and the mixture-dominated regime is significantly refined, so that each stepwise stability amplification \(\exp(\Lambda_k)\) remains controlled.

\begin{algorithm}[t]
\caption{\textsc{StabilityBalancedGrid} via cumulative-stability quantiles}
\label{alg:stability_balanced_grid}
\small
\begin{algorithmic}
\STATE \textbf{Input:} horizon \(T\), segment count \(n\), bound \(L:[0,T]\to\mathbb R_+\), auxiliary resolution \(M\), tolerance \(\tau_{\rm num}>0\)
\STATE \textbf{Output:} decreasing grid \(T=t_0\ge t_1\ge\cdots\ge t_n=0\)
\STATE Construct \(0=u_0<u_1<\cdots<u_M=T\) and evaluate \(L_i\gets L(u_i)\)
\STATE Set \(A_0\gets0\)
\FOR{\(i=1,\dots,M\)}
  \STATE \(A_i\gets A_{i-1}+\frac12(L_{i-1}+L_i)(u_i-u_{i-1})\)
\ENDFOR
\STATE \(A_T\gets A_M\)
\IF{\(A_T\le\tau_{\rm num}\)}
  \STATE Return the uniform grid \(t_k=(1-k/n)T\)
\ENDIF
\STATE Set \(\tau_0\gets0\) and \(\tau_n\gets T\)
\FOR{\(j=1,\dots,n-1\)}
  \STATE \(a_j\gets jA_T/n\)
  \STATE Find the smallest \(i\in\{1,\dots,M\}\) with \(A_i\ge a_j\)
  \IF{\(A_i>A_{i-1}\)}
    \STATE \(\tau_j\gets u_{i-1}+(u_i-u_{i-1})(a_j-A_{i-1})/(A_i-A_{i-1})\)
  \ELSE
    \STATE \(\tau_j\gets u_i\) \hfill\(\triangleright\) generalized-inverse convention on a plateau
  \ENDIF
\ENDFOR
\FOR{\(k=0,\dots,n\)}
  \STATE \(t_k\gets\tau_{n-k}\)
\ENDFOR
\STATE Return \(\{t_k\}_{k=0}^n\)
\end{algorithmic}
\end{algorithm}

\section{Experiments}
\label{sec:experiments}

\subsection{Validation of static score approximation}
\label{sec:exp_static}

In this subsection, we illustrate the approximation behavior described by~\autoref{thm:score_Lp_approx}. The implementation used for the experiments is available in the \href{https://github.com/GeminiLiming/quantitative-flow-distillation}{public code repository}. We consider the GMM--OU setup in~\autoref{sec:theoretical_setup}. The initial distribution is an isotropic Gaussian mixture in dimension \(d=64\) with \(K=20\) components and uniform weights \(\pi_k=1/K\). Mixture means \(\{\bm{\mu}_k\}_{k=1}^K\) are sampled i.i.d.\ from the uniform distribution on \([-5,5]^{64}\). The forward dynamics follow an OU process~\eqref{eq:ou_sde} with decay \(m_t=\exp(-\beta t)\) (we use \(\beta=5\) and \(T=1\)), so that each component remains Gaussian at time \(t\) with mean \(\bm m_k(t)=m_t\bm{\mu}_k\) and isotropic variance \(s_t^2=m_t^2\sigma^2+1-m_t^2\). We evaluate at normalized times \(t/T\in\{0.1,0.3,0.5,0.7,0.9\}\). Samples \(\bm X_t\sim p_t\) are generated directly from this closed-form marginal.

For each time \(t\), we train a separate unconditioned network \(\hat{\bm s}_t(\cdot)\) to approximate the score \(\nabla_{\bm x}\log p_t(\bm x)\). The architecture is a fully connected multilayer perceptron (MLP) with alternating nonlinearities (ReLU and ReQU) and \(Q=4\) hidden layers. Optimization uses AdamW~\cite{loshchilov2019decoupled} with learning rate \(2\times 10^{-4}\), batch size \(1024\), gradient clipping at \(1.0\), and a fixed number of \(20{,}000\) optimization steps. The ground-truth score \(\nabla_{\bm x}\log p_t(\bm x)\) is computed in closed form using~\autoref{prop:gmm_score_hess_closed}. We evaluate the relative squared error
\begin{equation*}
\mathrm{Rel\text{-}MSE}(t)
=
\frac{\mathbb{E}_{\bm X_t\sim p_t}\bigl[\|\hat{\bm s}_t(\bm X_t)-\nabla_{\bm x}\log p_t(\bm X_t)\|_2^2\bigr]}
{\mathbb{E}_{\bm X_t\sim p_t}\bigl[\|\nabla_{\bm x}\log p_t(\bm X_t)\|_2^2\bigr]},
\end{equation*}
where expectations are approximated by Monte Carlo using \(5\times 10^4\) samples from \(p_t\).

\paragraph{Empirical results.}
\autoref{tab:static_score_sigma_sweep} and~\autoref{tab:static_score_width_sweep} report the mean \(\mathrm{Rel\text{-}MSE}(t)\) over \(10\) independent trials for two sweeps. In the \emph{variance sweep} (\autoref{tab:static_score_sigma_sweep}), we fix the width \(G=512\) and vary the initial component scale \(\sigma\in\{0.01,0.1,1\}\). In the \emph{width sweep} (\autoref{tab:static_score_width_sweep}), we fix \(\sigma=0.1\) and vary the width \(G\in\{128,256,512,1024,2048\}\).

Across the sweeps, \(\mathrm{Rel\text{-}MSE}(t)\) generally decreases between \(t/T=0.1\) and \(t/T=0.9\), although it is not monotone at every intermediate time. This is consistent with later OU marginals being smoother and easier to approximate with a fixed-capacity MLP.
In the variance sweep, the dependence on \(\sigma\) is mild: all three choices achieve comparable error levels at each \(t/T\), and the overall variation across \(\sigma\) stays within the same order of magnitude.
This suggests that, in this regime, the main difficulty comes from learning the responsibility-weighted averaging structure of the mixture score, rather than from the specific initial component scale.

In contrast, the width sweep shows a clear capacity effect at early times.
At \(t/T\in\{0.1,0.3\}\), increasing width substantially reduces error (e.g., from the \(10^{-2}\) level at \(G=128\) to the \(10^{-3}\) level by \(G=2048\)), while the gains are noticeably smaller for \(t/T\ge 0.7\).
This pattern is consistent with the early-time score being sharper and requiring more capacity to resolve, whereas late-time scores are smoother and already well-approximated by moderate widths.
Overall, these results are consistent with the approximation guarantees in~\autoref{thm:score_Lp_approx}.

\begin{table}[htb]
\centering
\caption{Score approximation, variance sweep. Relative MSE of \(\hat{\bm s}_t\) at normalized times \(t/T\in\{0.1,0.3,0.5,0.7,0.9\}\). Entries report the mean over \(10\) trials.}
\label{tab:static_score_sigma_sweep}
\setlength{\tabcolsep}{6pt}
\small
\begin{tabular}{lccccc}
\toprule
\(\sigma\) & \(0.1\) & \(0.3\) & \(0.5\) & \(0.7\) & \(0.9\) \\
\midrule
\(0.01\) & \(5.53\times 10^{-3}\) & \(5.12\times 10^{-3}\) & \(3.32\times 10^{-3}\) & \(1.62\times 10^{-3}\) & \(1.59\times 10^{-3}\) \\
\(0.1\)  & \(6.96\times 10^{-3}\) & \(5.33\times 10^{-3}\) & \(3.15\times 10^{-3}\) & \(2.03\times 10^{-3}\) & \(1.65\times 10^{-3}\) \\
\(1\)    & \(3.87\times 10^{-3}\) & \(5.33\times 10^{-3}\) & \(3.30\times 10^{-3}\) & \(1.81\times 10^{-3}\) & \(1.73\times 10^{-3}\) \\
\bottomrule
\end{tabular}
\end{table}

\begin{table}[htb]
\centering
\caption{Score approximation with width sweep. Relative MSE of \(\hat{\bm s}_t\) at normalized times \(t/T\in\{0.1,0.3,0.5,0.7,0.9\}\) with \(\sigma=0.1\). Entries report the mean over \(10\) trials.}
\label{tab:static_score_width_sweep}
\setlength{\tabcolsep}{6pt}
\small
\begin{tabular}{lccccc}
\toprule
\(G\) & \(0.1\) & \(0.3\) & \(0.5\) & \(0.7\) & \(0.9\) \\
\midrule
\(128\)  & \(3.40\times 10^{-2}\) & \(3.98\times 10^{-2}\) & \(8.62\times 10^{-3}\) & \(5.47\times 10^{-3}\) & \(5.27\times 10^{-3}\) \\
\(256\)  & \(9.60\times 10^{-3}\) & \(1.90\times 10^{-2}\) & \(4.96\times 10^{-3}\) & \(2.99\times 10^{-3}\) & \(3.48\times 10^{-3}\) \\
\(512\)  & \(6.71\times 10^{-3}\) & \(4.93\times 10^{-3}\) & \(3.11\times 10^{-3}\) & \(1.49\times 10^{-3}\) & \(1.83\times 10^{-3}\) \\
\(1024\) & \(3.04\times 10^{-3}\) & \(3.15\times 10^{-3}\) & \(2.24\times 10^{-3}\) & \(1.15\times 10^{-3}\) & \(1.42\times 10^{-3}\) \\
\(2048\) & \(2.55\times 10^{-3}\) & \(2.28\times 10^{-3}\) & \(1.48\times 10^{-3}\) & \(1.27\times 10^{-3}\) & \(1.26\times 10^{-3}\) \\
\bottomrule
\end{tabular}
\end{table}

\subsection{Parameter-clipping sensitivity in a stiff flow-map regression}
\label{sec:exp_flow_map_bound}

In this subsection, we examine the empirical trend suggested by the stability-certificate comparison in~\autoref{thm:lipschitz_mismatch_distill}: tighter architectural parameter bounds can make the backward flow-map regression harder, particularly as the score field sharpens. The experiment is diagnostic and does not estimate the exact Lipschitz constants appearing in the theorem.

We use the GMM--OU setup in~\autoref{sec:exp_static}. We fix \(d=64\), \(K=20\), uniform weights \(\pi_k=1/K\), and an OU decay \(m_t=\exp(-\beta t)\) with \(\beta=5\) and \(T=1\). In contrast to~\autoref{sec:exp_static}, the mixture means are fixed across all trials and all configurations, i.e., we sample \(\{\bm{\mu}_k\}_{k=1}^K\) once from \(\mathrm{Unif}([-5,5]^{64})\) using a dedicated mean seed, and reuse these means for every run. The trial seed controls only stochasticity in sampling, minibatches, and model initialization. We set the evaluation time to the OU half-life
\begin{equation*}
\Delta=\frac{\log 2}{2\beta},
\end{equation*}
as in~\autoref{thm:lipschitz_mismatch_distill} and consider the backward probability-flow transport \(\Phi_{0\leftarrow\Delta}\). Training inputs are \(\bm X_\Delta\sim p_\Delta\), generated directly from the closed-form marginal.

We write the numerical teacher map as \(\widetilde\Phi^{\mathrm{RK4},10}_{0\leftarrow\Delta}\). It is computed by integrating the probability-flow ODE backward from \(t=\Delta\) to \(t=0\) using a fixed-step classical fourth-order Runge--Kutta (RK4) solver with \(10\) steps. The drift uses the closed-form score \(\nabla_{\bm x}\log p_t(\bm x)\) (computed exactly from the GMM marginal). All configurations in this sweep use the same discretized teacher; the reported regression errors are therefore relative to this numerical target and do not by themselves bound the integration error with respect to the exact continuous flow \(\Phi_{0\leftarrow\Delta}\). We learn a student map \(\Psi_\theta(\bm x_\Delta)\approx \widetilde\Phi^{\mathrm{RK4},10}_{0\leftarrow\Delta}(\bm x_\Delta)\) using the same MLP family as in~\autoref{sec:exp_static}, with alternating ReLU/ReQU nonlinearities, depth \(Q=4\), and width \(G=512\). To vary one architectural factor entering the local Lipschitz upper bound, after \emph{every} optimizer step we hard-clip every linear weight and bias to \([-B,B]\), and sweep
\begin{equation*}
B\in\{0.001,0.01,0.1,0.25\}.
\end{equation*}
This clipping controls the parameter factor \(B\) in~\autoref{thm:lipschitz_mismatch_distill}. Because ReQU is not globally Lipschitz, it is only a proxy for local Lipschitz capacity unless the hidden preactivation ranges are controlled as well.

We control the stiffness by sweeping \(\sigma\) over \(\{0.1, 0.01, 0.001, 0.0001\}\). For each \((\sigma,B)\), we train by regression on pairs \((\bm X_\Delta,\widetilde\Phi^{\mathrm{RK4},10}_{0\leftarrow\Delta}(\bm X_\Delta))\) using AdamW with learning rate \(2\times 10^{-4}\), batch size \(1024\), gradient clipping \(1.0\), and a fixed training budget of \(20{,}000\) optimizer steps. Each \((\sigma,B)\) configuration is repeated over \(10\) trial seeds. We report the relative squared error of the map approximation
\begin{equation*}
\mathrm{Rel\text{-}MSE}_{\mathrm{map}}
=
\frac{\mathbb{E}_{\bm X_\Delta\sim p_\Delta}\bigl[\|\Psi_\theta(\bm X_\Delta)-\widetilde\Phi^{\mathrm{RK4},10}_{0\leftarrow\Delta}(\bm X_\Delta)\|_2^2\bigr]}
{\mathbb{E}_{\bm X_\Delta\sim p_\Delta}\bigl[\|\widetilde\Phi^{\mathrm{RK4},10}_{0\leftarrow\Delta}(\bm X_\Delta)\|_2^2\bigr]},
\end{equation*}
where the expectation is estimated by Monte Carlo using \(5\times 10^4\) fresh samples \(\bm X_\Delta\sim p_\Delta\). We report the mean across trial seeds.

\paragraph{Empirical results.}
\autoref{tab:flow_map_bound_B_sweep} shows a clear parameter-clipping threshold that becomes more stringent as the score field sharpens.
When the clipping bound is extremely tight (\(B=0.001\)), the student fails to approximate \(\widetilde\Phi^{\mathrm{RK4},10}_{0\leftarrow\Delta}\) across all noise levels, with \(\mathrm{Rel\text{-}MSE}_{\mathrm{map}}\approx 1\).
For the least stiff regime \(\sigma=0.1\), a modest parameter bound already suffices, with error dropping from \(9.44\times 10^{-1}\) at \(B=0.001\) to \(3.98\times 10^{-3}\) at \(B=0.01\), and further to \(2.49\times 10^{-5}\) at \(B=0.1\).
In contrast, for \(\sigma\in\{0.01,0.001\}\), the same \(B=0.01\) parameter bound is no longer adequate, leaving errors at \(7.61\times 10^{-1}\) and \(9.56\times 10^{-1}\), respectively, while increasing to \(B\ge 0.1\) reduces the error to the \(10^{-3}\) level.
In the stiffest setting \(\sigma=10^{-4}\), even \(B=0.25\) only reaches \(2.05\times 10^{-2}\), indicating that the parameter capacity required by this trained architecture and the overall difficulty increase sharply in the low-noise regime, consistent with the stability-certificate trend in~\autoref{thm:lipschitz_mismatch_distill}.

\begin{table}[htb]
\centering
\small
\caption{Mean flow-map relative MSE over \(10\) trials under parameter clipping, for the displayed values of \(\sigma\) and \(B\).}
\label{tab:flow_map_bound_B_sweep}
\begin{tabular}{lcccc}
\toprule
\(\sigma\) & \(B=0.001\) & \(B=0.01\) & \(B=0.1\) & \(B=0.25\) \\
\midrule
\(0.1\)    & \(9.44\times 10^{-1}\) & \(3.98\times 10^{-3}\) & \(2.49\times 10^{-5}\) & \(2.58\times 10^{-5}\) \\
\(0.01\)   & \(9.88\times 10^{-1}\) & \(7.61\times 10^{-1}\) & \(6.46\times 10^{-3}\) & \(6.67\times 10^{-3}\) \\
\(0.001\)  & \(1.00\times 10^{0}\)  & \(9.56\times 10^{-1}\) & \(8.14\times 10^{-3}\) & \(7.76\times 10^{-3}\) \\
\(0.0001\) & \(1.00\times 10^{0}\)  & \(9.86\times 10^{-1}\) & \(3.16\times 10^{-2}\) & \(2.05\times 10^{-2}\) \\
\bottomrule
\end{tabular}
\end{table}

\subsection{Transformer architectures}
\label{sec:exp_residual_vs_transformer}

In this subsection, we examine a qualitative implication of the compositional viewpoint in~\autoref{thm:resnet_flow_tail_approx} by evaluating how approximation accuracy for a numerical approximation of \(\Phi_{0\leftarrow T}\) scales with Transformer depth. Since the theorem is stated for clipped ReLU--ReQU residual compositions, this experiment is not a direct test of its formal hypotheses. We use the same GMM--OU setup as in~\autoref{sec:exp_flow_map_bound} with fixed mixture means shared across all runs, and fix \(T=1\), \(\beta=5\), and \(\sigma=10^{-4}\). Training data consist of pairs \((\bm X_T,\widetilde\Phi^{\mathrm{RK4},100}_{0\leftarrow T}(\bm X_T))\). We sample \(\bm X_T\sim p_T\) from the closed-form GMM--OU marginal, then compute the numerical teacher map \(\widetilde\Phi^{\mathrm{RK4},100}_{0\leftarrow T}\) by integrating the probability-flow ODE backward from \(t=T\) to \(t=0\) with a fixed-step RK4 solver using \(100\) steps. All depths are compared against this common discretized teacher, so the experiment does not independently estimate the integration error relative to the exact continuous flow \(\Phi_{0\leftarrow T}\).

We evaluate a Transformer~\cite{vaswani2017attention} parameterization that treats \(\bm x\in\mathbb R^d\) as a length-\(d\) token sequence. Each coordinate \(x_i\) is embedded to a \(d_{\text{model}}\)-dimensional token, augmented by a learned positional embedding, processed by \(L\) self-attention blocks, and projected back to \(\mathbb R^d\). The model uses standard pre-normalization, residual connections, and a positionwise feedforward sublayer. We sweep the number of attention blocks \(L\in\{1,2,4,8,16\}\) and keep all other Transformer hyperparameters fixed across runs. All student models are trained end-to-end by minimizing mean squared error between the student prediction and the teacher map output using AdamW with learning rate \(2\times 10^{-4}\), batch size \(1024\), gradient clipping \(1.0\), and \(20{,}000\) optimizer steps. Each configuration is repeated over \(10\) trial seeds. For a student map \(\mathcal S\) and its numerical teacher \(\mathcal T\), define the common end-to-end metric
\begin{equation}
\label{eq:rel_mse_e2e}
\mathrm{Rel\text{-}MSE}_{\mathrm{e2e}}(\mathcal S,\mathcal T)
=
\frac{\mathbb E_{\bm X_T\sim p_T}\bigl[\|\mathcal S(\bm X_T)-\mathcal T(\bm X_T)\|_2^2\bigr]}
{\mathbb E_{\bm X_T\sim p_T}\bigl[\|\mathcal T(\bm X_T)\|_2^2\bigr]}.
\end{equation}
For this experiment, \(\mathcal S=\Psi\) and \(\mathcal T=\widetilde\Phi^{\mathrm{RK4},100}_{0\leftarrow T}\); the expectations are estimated using \(5\times 10^4\) samples \(\bm X_T\sim p_T\).

\paragraph{Empirical results.}
\autoref{tab:transformer_depth} reports the mean \(\mathrm{Rel\text{-}MSE}_{\mathrm{e2e}}\) over \(10\) trials for Transformer students with varying depth. The approximation error decreases as the number of attention blocks \(L\) increases, from \(7.10\times 10^{-2}\) at \(L=1\) to \(3.03\times 10^{-2}\) at \(L=16\), which is a \(2.35\times\) reduction. The gains are modest from \(L\in\{1,2,4\}\), while deeper models yield a more pronounced improvement, consistent with the view that increased depth better matches the compositional structure of the underlying flow map.

\begin{table}[htb]
\centering
\caption{Transformer approximation of the numerical teacher \(\widetilde\Phi^{\mathrm{RK4},100}_{0\leftarrow T}\) in the GMM--OU setting. Entries report the mean \(\mathrm{Rel\text{-}MSE}_{\mathrm{e2e}}\) over \(10\) trials.}
\label{tab:transformer_depth}
\begin{tabular}{cc}
\toprule
Number of blocks \(L\) & Mean \(\mathrm{Rel\text{-}MSE}_{\mathrm{e2e}}\) \\
\midrule
1  & \(7.10\times 10^{-2}\) \\
2  & \(7.04\times 10^{-2}\) \\
4  & \(6.74\times 10^{-2}\) \\
8  & \(4.70\times 10^{-2}\) \\
16 & \(3.03\times 10^{-2}\) \\
\bottomrule
\end{tabular}
\end{table}

\subsection{Stability-balanced discretization vs.~uniform grids}
\label{sec:exp_stab_grid_vs_uniform}

In this subsection, we empirically illustrate the stability-balancing principle in~\autoref{subsec:stability_balanced_time_discretization}. The analysis assigns a larger worst-case error multiplier to segments with large stability exponent \(\Lambda_k\), as defined above. We therefore compare a uniform time grid against the stability-balanced grid produced by~\autoref{alg:stability_balanced_grid}, which approximately equalizes \(\Lambda_k\) by taking a uniform partition in the cumulative stability coordinate \(A(t)=\int_0^t L(u)\mathrm du\) as in~\eqref{eq:A_uniform_partition}. We use the same GMM--OU setup as in~\autoref{sec:exp_static}. We fix \(d=64\), \(K=20\), uniform weights \(\pi_k=1/K\), and draw mixture means once and then keep them fixed across all runs. The OU decay is \(m_t=\exp(-\beta t)\) with \(\beta=5\) and \(T=1\). The stability profile \(L(t)\) is computed from the closed-form quantities in the GMM--OU model as in~\eqref{eq:L_profile_distill}, and its cumulative \(A(t)\) is computed numerically to construct the stability-balanced grid via~\autoref{alg:stability_balanced_grid}.

For a fixed number of segments \(n\), we compare a uniform grid and a stability-balanced grid, both written in decreasing time order
\begin{equation*}
T=t^{\mathrm{uni}}_0>t^{\mathrm{uni}}_1>\cdots>t^{\mathrm{uni}}_n=0,
\qquad
T=t^{\mathrm{sb}}_0>t^{\mathrm{sb}}_1>\cdots>t^{\mathrm{sb}}_n=0.
\end{equation*}
The uniform grid is
\begin{equation*}
t^{\mathrm{uni}}_k=\Bigl(1-\frac{k}{n}\Bigr)T,\quad k=0,1,\dots,n,
\end{equation*}
while the stability-balanced grid \(\{t^{\mathrm{sb}}_k\}_{k=0}^n\) is defined implicitly by
\begin{equation*}
A\bigl(t^{\mathrm{sb}}_k\bigr)=\frac{n-k}{n}A(T),\quad k=0,1,\dots,n,
\end{equation*}
and computed by~\autoref{alg:stability_balanced_grid}. Both grids share the same endpoints and segment count \(n\).

On either grid, we distill each numerical segment map \(\widetilde\Phi^{\mathrm{RK4}}_{t_k\leftarrow t_{k-1}}\) independently by regression. For segment \(k\), training inputs are \(\bm X_{t_{k-1}}\sim p_{t_{k-1}}\) generated from the closed-form marginal, and the teacher target \(\widetilde\Phi^{\mathrm{RK4}}_{t_k\leftarrow t_{k-1}}(\bm X_{t_{k-1}})\) is computed by RK4 integration of the probability-flow ODE over \([t_k,t_{k-1}]\) using the closed-form score. Both grids are evaluated against the same discretized-teacher construction; the comparison therefore isolates the effect of the grid but does not estimate integration error relative to the exact continuous flow. The student for each segment is an MLP with ReLU activations, depth \(Q=4\) and width \(G=512\). We use AdamW with learning rate \(2\times 10^{-4}\), batch size \(1024\), and gradient clipping \(1.0\). For each fixed \(n\), we use the same total optimization budget for both grids and split that budget uniformly across the \(n\) segment students. The distilled full map is the composition \(\Psi_{0\leftarrow T}=\Psi_n\circ\cdots\circ\Psi_1\). The experiment is repeated for \(10\) trial seeds. We evaluate~\eqref{eq:rel_mse_e2e} with \(\mathcal S=\Psi_{0\leftarrow T}\) and the common numerical reference \(\mathcal T=\widetilde\Phi^{\mathrm{RK4}}_{0\leftarrow T}\), using \(5\times 10^4\) samples \(\bm X_T\sim p_T\).

\paragraph{Empirical results.}
\autoref{tab:stab_grid_vs_uniform} compares the uniform grid and the stability-balanced grid at \(\sigma=1\). Across all segment counts \(n\in\{4,8,16\}\), the stability-balanced grid achieves a lower end-to-end error, with the largest gains at coarse discretizations. In particular, it reduces \(\mathrm{Rel\text{-}MSE}_{\mathrm{e2e}}\) by about \(37.7\%\) at \(n=4\) and \(51.9\%\) at \(n=8\), consistent with the stability-balancing mechanism that avoids assigning a disproportionately large certified multiplier \(e^{\Lambda_k}\) to any single segment. As \(n\) increases from \(4\) to \(16\), the end-to-end error increases for both grids, which matches the training protocol that fixes a total optimization budget and splits it across segments, so each one-segment student receives fewer updates and approximation errors compound over a longer composition chain.

\begin{table}[htb]
\centering
\small
\caption{Stability-balanced discretization vs.~uniform grids. Relative end-to-end MSE \(\mathrm{Rel\text{-}MSE}_{\mathrm{e2e}}\) for the composition of \(n\) independently distilled one-segment maps at \(\sigma=1\), with \(n\in\{4,8,16\}\). Entries report the mean over \(10\) trials.}
\label{tab:stab_grid_vs_uniform}
\begin{tabular}{lccc}
\toprule
\(n\) & Uniform & Stability-balanced & Reduction in mean\\
\midrule
\(4\)  & \(4.10\times 10^{-2}\) & \(2.55\times 10^{-2}\) & \(37.7\%\)\\
\(8\)  & \(8.87\times 10^{-2}\) & \(4.27\times 10^{-2}\) & \(51.9\%\)\\
\(16\) & \(1.67\times 10^{-1}\) & \(1.30\times 10^{-1}\) & \(22.2\%\)\\
\bottomrule
\end{tabular}
\end{table}

\section{Conclusion}
\label{sec:conclusion}

This work develops an approximation framework for trajectory distillation in the isotropic GMM--OU setting. The score is approximated constructively by ReLU--ReQU networks, while the propagation of local errors is controlled by an explicit time-integrated Jacobian bound. The comparison in~\autoref{thm:lipschitz_mismatch_distill} is deliberately certificate-level: it identifies when the computed Gr\"onwall factor exceeds a certified local architectural budget, without claiming an approximation lower bound. For every fixed horizon \(T\), a clipped residual construction approximates the full flow with \(\widetilde{\mathcal O}_{T,\mathrm{model}}(\varepsilon^{-1})\) blocks and \(O(e^{\Lambda_T}\varepsilon)\) error, and the cumulative-stability coordinate yields a robust non-uniform grid. The experiments illustrate the usefulness of this certificate for architecture and grid design. Extensions to non-isotropic mixtures, sharper lower bounds for actual flow-map complexity, and learnability of the constructed networks remain open directions.

\newpage
\appendix
\phantomsection
\label{allapp}
\section[Proof of the score-approximation theorem]{Proof of~\autoref{thm:score_Lp_approx}}
\label{appsec:proof_of_approx}

This section provides a detailed proof of~\autoref{thm:score_Lp_approx}. We first construct scalar networks approximating the OU coefficients and use them to approximate the affine logits of the Gaussian-mixture posterior. We then control the resulting logit error in \(L^p(p_t)\), construct a compact-domain softmax approximation, and combine these ingredients with the explicit score representation to obtain the final \(L^p\)-error and network-complexity bounds.

\subsection{Stabilization and elementary network modules}

For the affine logits \(b_k\) in~\eqref{eq:affine_logit_def}, define
\begin{equation*}
y_k(\bm x,t)
\coloneqq
b_k(\bm x,t)-\max_{1\le j\le K}b_j(\bm x,t),
\qquad
\widetilde y_k(\bm x,t)
\coloneqq
\max\{y_k(\bm x,t),-M\}.
\end{equation*}
Then \(y_k\le0\), \(\max_k y_k=0\), and the posterior vector is \(\bm\gamma=\operatorname{softmax}(\bm y)\). Let
\begin{equation*}
S_i^{(M)}(\widetilde{\bm y})
\coloneqq
\frac{e^{\widetilde y_i}}{\sum_{j=1}^K e^{\widetilde y_j}}.
\end{equation*}

\begin{lemma}
\label{lem:softmax_tail}
For every \(M>0\),
\begin{equation}
\label{eq:softmax_tail_revised}
\Bigl\lVert S^{(M)}(\widetilde{\bm y})-\operatorname{softmax}(\bm y)\Bigr\rVert_1
\le 2Ke^{-M}.
\end{equation}
\end{lemma}

\begin{proof}
Set \(Z=\sum_j e^{y_j}\) and \(\widetilde Z=\sum_j e^{\widetilde y_j}\). Since one coordinate of each vector equals zero, \(Z,\widetilde Z\ge1\). Therefore
\begin{equation*}
\Bigl\lVert S^{(M)}(\widetilde{\bm y})-\operatorname{softmax}(\bm y)\Bigr\rVert_1
\le
\frac{|Z-\widetilde Z|}{Z}
+
\frac{\sum_k|e^{\widetilde y_k}-e^{y_k}|}{Z}
\le
2\sum_{k=1}^K|e^{\widetilde y_k}-e^{y_k}|.
\end{equation*}
A summand is zero when \(y_k\ge-M\), and is at most \(e^{-M}\) otherwise. This proves~\eqref{eq:softmax_tail_revised}.
\end{proof}

\begin{lemma}
\label{lem:exact_relu_requ_modules}
The maps \(u\mapsto u^2\) and \((u,v)\mapsto uv\) are exactly realizable by one-hidden-layer ReQU networks of widths \(2\) and \(6\), respectively. The map \((z_1,\ldots,z_K)\mapsto\max_k z_k\) is exactly realizable by a ReLU binary tree of depth \(\lceil\log_2K\rceil\) and width at most \(3K\). The maps \(u\mapsto\max\{u,-M\}\) and coordinatewise clipping to a compact interval are exactly realizable by constant-depth ReLU networks. Identity channels can be propagated exactly through either activation type with constant width.
\end{lemma}

\begin{proof}
Use
\begin{equation*}
u^2=(u_+)^2+((-u)_+)^2,
\qquad
uv=\frac12\bigl((u+v)^2-u^2-v^2\bigr),
\end{equation*}
and
\begin{equation*}
\max\{a,b\}=\frac12\bigl(a+b+\operatorname{ReLU}(a-b)+\operatorname{ReLU}(b-a)\bigr).
\end{equation*}
Iterating the binary maximum in a balanced tree gives the stated depth and width. Finally,
\(\max\{u,-M\}=-M+\operatorname{ReLU}(u+M)\), and two such one-sided truncations give clipping to an interval. Identity propagation through a ReLU layer uses \(u=\operatorname{ReLU}(u)-\operatorname{ReLU}(-u)\); through a ReQU layer it uses
\[
u=\frac14\bigl((u+1)^2-(u-1)^2\bigr),
\]
with each full square represented by two ReQU units. These identities also allow subnetworks with different depths to be padded before parallelization.
\end{proof}

\begin{lemma}
\label{lem:parallel_identity_padding}
Let \(\mathcal N_i\), \(i=1,\ldots,J\), be finitely many ReLU--ReQU subnetworks. Denote their output dimensions, depths, widths, sparsities, and parameter bounds by \(m_i,Q_i,G_i,S_i,B_i\), respectively. Let \(Q\ge \max_i Q_i\). After synchronizing activation patterns using the exact identity modules from~\autoref{lem:exact_relu_requ_modules}, their parallel concatenation can be realized by a ReLU--ReQU network satisfying \(Q_{\mathrm{par}}\le C Q,G_{\mathrm{par}}\le C\sum_{i=1}^J(G_i+m_i),S_{\mathrm{par}} \le C\sum_{i=1}^J(S_i+m_iQ),B_{\mathrm{par}} \le C\max\{1,\max_{1\le i\le J}B_i\},\) where \(C\) is universal. If the branchwise activation patterns are already synchronized, the padding term \(m_iQ\) may be replaced by \(m_i(Q-Q_i)\). The same bounds hold, up to another universal constant, after concatenating a fixed number of exact affine, square, product, maximum, or clipping modules.
\end{lemma}

\begin{proof}
At every common layer, place the affine maps of the active branches on the diagonal of a block matrix. Whenever a branch has no prescribed operation at that common layer, continue it by the exact identity representation corresponding to the required ReLU or ReQU activation. Each inserted identity layer uses \(O(m_i)\) nonzero parameters and coefficients of absolute value bounded by a universal constant. At most \(CQ\) such layers are needed per branch, giving the term \(O(m_iQ)\); when the activation patterns already agree, only the \(Q-Q_i\) terminal padding layers are needed. Block-diagonal parallelization adds widths and sparsities, while the largest parameter magnitude is the maximum of the branchwise bounds and the constant coefficients used by the exact modules. Concatenating a fixed number of the elementary modules changes all estimates by only a universal factor.
\end{proof}

\subsection{Scalar OU coefficients and affine logits}
\begin{lemma}
\label{lem:scalar_ou_coefficients}
Assume~\autoref{ass:suzuki_schedule} and let \(r_t\coloneqq s_t^{-2} = \bigl(1+(\sigma^2-1)m_t^2\bigr)^{-1}.\) For every \(0<\delta<1/2\), there exist one-dimensional ReLU--ReQU networks \(\widehat m_\delta\) and \(\widehat r_\delta\) such that
\begin{equation}
\label{eq:scalar_ou_approx}
\sup_{t\in[0,T]}
\max\{
|\widehat m_\delta(t)-m_t|,
|\widehat r_\delta(t)-r_t|
\}
\le \delta.
\end{equation}

Define \(A_{\delta,\sigma} \coloneqq 2+\delta^{-1}+\sigma+\sigma^{-1}.\) Writing \(Q_\delta,G_\delta,S_\delta,B_\delta\) for the maximum depth, width, sparsity, and parameter bound of the two networks, respectively, they may be chosen such that
\[
Q_\delta
\le
C\bigl(1+\log A_{\delta,\sigma}\bigr)^2,
\qquad
G_\delta
\le
C\bigl(1+\log A_{\delta,\sigma}\bigr)^3,
\qquad
S_\delta
\le
C\bigl(1+\log A_{\delta,\sigma}\bigr)^4,
\]
and
\[
B_\delta
\le
C\bigl(
1+\delta^{-2}+\sigma^4+\sigma^{-4}
\bigr)
\le
C A_{\delta,\sigma}^4.
\]
Here \(C>0\) depends only on \(T\) and the schedule constants in
\autoref{ass:suzuki_schedule}, but is independent of
\(\delta\) and \(\sigma\).

More specifically, the subnetwork approximating \(m_t\) satisfies
\[
Q_m,G_m
=
O\bigl(\log^2(\delta_m^{-1})\bigr),
\qquad
S_m
=
O\bigl(\log^3(\delta_m^{-1})\bigr),
\qquad
B_m
=
O\bigl(\log(\delta_m^{-1})\bigr),
\]
where \(\delta_m=c_\sigma\delta\) and \(c_\sigma\in(0,1]\) is
specified in the proof.
\end{lemma}

\begin{proof}
Set \(s_-^2\coloneqq\min\{1,\sigma^2\}, s_+^2\coloneqq\max\{1,\sigma^2\}.\) If \(\sigma\neq1\), define \(c_\sigma \coloneqq \min\{ 1, \frac{s_-^4}{4|\sigma^2-1|} \}, \delta_m\coloneqq c_\sigma\delta.\)
If \(\sigma=1\), set \(c_\sigma=1\) and
\(\delta_m=\delta\). By construction, \(c_\sigma^{-1} \le C\bigl(1+\sigma^2+\sigma^{-4}\bigr),\) and therefore \(\log(\delta_m^{-1}) \le C\log A_{\delta,\sigma}.\)

Applying Lemma~B.1 in~\cite{oko2023diffusion} with accuracy \(\delta_m\), we obtain a one-dimensional ReLU network approximating \(m_t\) uniformly on \([0,T]\) with error at most \(\delta_m\), and satisfying
\[
Q_m,G_m
=
O\bigl(\log^2(\delta_m^{-1})\bigr),
\quad
S_m
=
O\bigl(\log^3(\delta_m^{-1})\bigr), 
\quad
B_m 
= 
O\bigl(\log(\delta_m^{-1})\bigr).
\]

Clipping its output to \([0,1]\) cannot increase the uniform approximation error. Denote the resulting network by
\(\widehat m_\delta\). Since \(\delta_m\le\delta\),
\[
\sup_{t\in[0,T]}
|\widehat m_\delta(t)-m_t|
\le
\delta_m
\le
\delta.
\]

Using the exact ReQU square module, define \(\widehat q(t) \coloneqq 1+(\sigma^2-1)\widehat m_\delta(t)^2.\) Since \(0\le\widehat m_\delta(t)\le1\), it follows that \(\widehat q(t)\in[s_-^2,s_+^2].\) Moreover, because
\(0\le m_t,\widehat m_\delta(t)\le1\),
\[
|\widehat q(t)-s_t^2| = |\sigma^2-1| \bigl| \widehat m_\delta(t)^2-m_t^2 \bigr| \le 2|\sigma^2-1| |\widehat m_\delta(t)-m_t| \le 2|\sigma^2-1|\delta_m.
\]

If \(\sigma=1\), then \(s_t^2\equiv1\), and we simply take \(\widehat r_\delta\equiv1\). Suppose henceforth that \(\sigma\neq1\). Define \(\varepsilon_{\mathrm{rec}} \coloneqq \min\bigl\{ \frac{\delta}{2}, s_-^2, (s_+^2)^{-1} \bigr\}.\) Then \([s_-^2,s_+^2] \subset [\varepsilon_{\mathrm{rec}}, \varepsilon_{\mathrm{rec}}^{-1}].\)

By Lemma~F.7 in~\cite{oko2023diffusion}, there exists a one-dimensional ReLU network \(\phi_{\mathrm{rec}}\) satisfying
\[
\sup_{z\in
[\varepsilon_{\mathrm{rec}},
\varepsilon_{\mathrm{rec}}^{-1}]}
\bigl|
\phi_{\mathrm{rec}}(z)-z^{-1}
\bigr|
\le
\varepsilon_{\mathrm{rec}},
\]
whose depth, width, and sparsity are bounded by
\(C(1+\log(\varepsilon_{\mathrm{rec}}^{-1}))^2\),
\(C(1+\log(\varepsilon_{\mathrm{rec}}^{-1}))^3\), and
\(C(1+\log(\varepsilon_{\mathrm{rec}}^{-1}))^4\), respectively, and whose parameter
bound satisfies \(B_{\mathrm{rec}} \le C\varepsilon_{\mathrm{rec}}^{-2}.\) Define \(\widetilde r_\delta(t) \coloneqq \phi_{\mathrm{rec}}(\widehat q(t)).\) Since \(\widehat q(t)\in[s_-^2,s_+^2]\), \(\bigl| \widetilde r_\delta(t)-\widehat q(t)^{-1} \bigr| \le \varepsilon_{\mathrm{rec}} \le \frac{\delta}{2}.\) The map \(z\mapsto z^{-1}\) is \(s_-^{-4}\)-Lipschitz on \([s_-^2,s_+^2]\). By the definition of \(c_\sigma\), \(2s_-^{-4}|\sigma^2-1|c_\sigma \le \frac12.\) Consequently,
\[
\begin{aligned}
\bigl|
\widetilde r_\delta(t)-s_t^{-2}
\bigr|
&\le
\bigl|
\widetilde r_\delta(t)-\widehat q(t)^{-1}
\bigr|
+
\bigl|
\widehat q(t)^{-1}-s_t^{-2}
\bigr| \\
&\le
\frac{\delta}{2}
+
s_-^{-4}|\widehat q(t)-s_t^2| \\
&\le
\frac{\delta}{2}
+
2s_-^{-4}|\sigma^2-1|c_\sigma\delta \\
&\le
\delta.
\end{aligned}
\]

Finally, define \(\widehat r_\delta\) by clipping \(\widetilde r_\delta\) to \(\Bigl[ \frac{1}{2s_+^2}, \frac{2}{s_-^2} \Bigr].\) Since \(r_t=s_t^{-2} \in \Bigl[ \frac{1}{s_+^2},\frac{1}{s_-^2} \Bigr],\) this clipping cannot increase the approximation error. Hence \eqref{eq:scalar_ou_approx} holds.

It remains to verify the network complexity. By definition, \(\varepsilon_{\mathrm{rec}}^{-1} = \max\bigl\{ 2\delta^{-1}, s_-^{-2}, s_+^2 \bigr\},\) and therefore \(\log(\varepsilon_{\mathrm{rec}}^{-1}) \le C\log A_{\delta,\sigma}.\) The \(m_t\)-network has depth and width of order \(\log^2(\delta_m^{-1})\) and sparsity of order \(\log^3(\delta_m^{-1})\), while the reciprocal module has depth, width, and sparsity bounded respectively by \(C(1+\log\varepsilon_{\mathrm{rec}}^{-1})^2\), \(C(1+\log\varepsilon_{\mathrm{rec}}^{-1})^3\), and \(C(1+\log\varepsilon_{\mathrm{rec}}^{-1})^4\). Combining these networks with the constant-depth exact square and clipping modules gives
\[
Q_\delta
\le
C\bigl(1+\log A_{\delta,\sigma}\bigr)^2,
\qquad
G_\delta
\le
C\bigl(1+\log A_{\delta,\sigma}\bigr)^3,
\qquad
S_\delta
\le
C\bigl(1+\log A_{\delta,\sigma}\bigr)^4.
\]

Moreover,
\[
\begin{aligned}
\varepsilon_{\mathrm{rec}}^{-2}
=
\max\bigl\{
4\delta^{-2},
s_-^{-4},
s_+^4
\bigr\}
\le
C\bigl(
1+\delta^{-2}+\sigma^4+\sigma^{-4}
\bigr).
\end{aligned}
\]
The affine map \(u\longmapsto 1+(\sigma^2-1)u\) and the final clipping module have parameter bounds controlled by \(C(1+\sigma^2+\sigma^{-2})\). Thus, the complete construction satisfies
\[
B_\delta
\le
C\bigl(
1+\delta^{-2}+\sigma^4+\sigma^{-4}
\bigr)
\le
C A_{\delta,\sigma}^4.
\]
This completes the proof.
\end{proof}

\begin{lemma}
\label{lem:gmm_uniform_moment}
Let \(p\in[1,\infty)\) and \(\bm X_t\sim p_t\). Then
\begin{equation}
\label{eq:gmm_uniform_moment}
\sup_{t\in[0,T]}
\bigl\|\lVert\bm X_t\rVert_2\bigr\|_{L^p}
\le
\mu_{\max}+C s_+(\sqrt d+\sqrt p),
\end{equation}
where \(C\) is universal.
\end{lemma}

\begin{proof}
Conditionally on the component index \(J\),
\(\bm X_t=m_t\bm\mu_J+s_t\bm Z\) with \(\bm Z\sim\mathcal N(0,I_d)\). Minkowski's inequality, \(m_t\le1\), \(s_t\le s_+\), and the standard Gaussian moment estimate \(\bigl\|\lVert\bm Z\rVert_2\bigr\|_{L^p}\le C(\sqrt d+\sqrt p)\) give~\eqref{eq:gmm_uniform_moment}.
\end{proof}

\begin{lemma}
\label{lem:affine_logit_approx}
 Fix \(M>0\), \(0<\delta<1/2\), and \(p\in[1,\infty)\). Under~\autoref{ass:suzuki_schedule} and~\autoref{ass:gmm_ou_init}, there exists a ReLU--ReQU network
\[
\mathcal N_{\mathrm{logit},\delta}
\colon
\mathbb R^d\times[0,T]\to[-M,0]^K
\]
with output \(\widehat{\widetilde{\bm y}}\) satisfying \(\max_{1\le k\le K}\widehat{\widetilde y}_k(\bm x,t)=0\) for every \((\bm x,t)\in\mathbb R^d\times[0,T]\), and
\begin{equation}
\label{eq:affine_logit_Lp}
\max_{1\le k\le K}
\Bigl\|
\widehat{\widetilde y}_k(\bm X_t,t)
-
\widetilde y_k(\bm X_t,t)
\Bigr\|_{L^p(p_t)}
\le
C_{\sigma,\mu}\delta D_p,
\end{equation}
uniformly in \(t\in[0,T]\), where \(D_p \coloneqq 1+\mu_{\max}+s_+(\sqrt d+\sqrt p), C_{\sigma,\mu} \le C(1+s_-^{-2})(1+\mu_{\max})^2.\) Here \(s_-^2=\min\{1,\sigma^2\}\), \(s_+^2=\max\{1,\sigma^2\}\), and \(C\) depends only on \(T\) and the schedule constants in~\autoref{ass:suzuki_schedule}.

With \(A_{\delta,\sigma}=2+\delta^{-1}+\sigma+\sigma^{-1}\), the network may be chosen so that
\[
Q_{\mathrm{logit}}
\le
C\Bigl(\log K+\bigl(1+\log A_{\delta,\sigma}\bigr)^2\Bigr),
\]
\[
G_{\mathrm{logit}}
\le
C\Bigl(dK+\bigl(1+\log A_{\delta,\sigma}\bigr)^3\Bigr),
\]
\[
\begin{aligned}
S_{\mathrm{logit}}
&\le
C\Bigl[
 dK+(d+K)\Bigl(
 \log K
 +\bigl(1+\log A_{\delta,\sigma}\bigr)^4
 \Bigr)
\Bigr].
\end{aligned}
\]
and
\[
B_{\mathrm{logit}}
\le
C\bigl(
1+M+\log(\pi_{\min}^{-1})
+\mu_{\max}+\mu_{\max}^2
+\delta^{-2}+\sigma^4+\sigma^{-4}
\bigr).
\]
\end{lemma}

\begin{proof}
Take \(\widehat m_\delta\) and \(\widehat r_\delta\) from~\autoref{lem:scalar_ou_coefficients}, and form exactly \(\widehat c(t) = \widehat m_\delta(t)\widehat r_\delta(t), \widehat d(t) = \frac12\widehat m_\delta(t)^2\widehat r_\delta(t).\) The corresponding exact coefficients are \(c_t=m_tr_t, d_t=\frac12m_t^2r_t.\) By the clipping used in~\autoref{lem:scalar_ou_coefficients},
\[
0\le\widehat m_\delta(t),m_t\le1,
\qquad
|\widehat r_\delta(t)|\le\frac{2}{s_-^2},
\qquad
|r_t|\le\frac{1}{s_-^2}.
\]
Consequently,
\[
|\widehat c(t)-c_t|
\le
|\widehat m_\delta(t)-m_t||\widehat r_\delta(t)|
+m_t|\widehat r_\delta(t)-r_t|
\le
(1+2s_-^{-2})\delta,
\]
and
\[
|\widehat d(t)-d_t|
\le
\frac12|\widehat m_\delta(t)^2-m_t^2||\widehat r_\delta(t)|
+\frac12m_t^2|\widehat r_\delta(t)-r_t|
\le
\biggl(\frac12+2s_-^{-2}\biggr)\delta.
\]
Thus
\begin{equation}
\label{eq:coefficient_approx}
|\widehat c(t)-c_t|+|\widehat d(t)-d_t|
\le
C(1+s_-^{-2})\delta.
\end{equation}

For \(k=1,\ldots,K\), define
\[
\widehat b_k(\bm x,t)
=
\log\pi_k
+
\widehat c(t)\bm\mu_k^\top\bm x
-
\widehat d(t)\lVert\bm\mu_k\rVert_2^2.
\]
The dot products are affine, and the remaining scalar products are realized exactly using~\autoref{lem:exact_relu_requ_modules}. Since
\[
b_k(\bm x,t)
=
\log\pi_k+c_t\bm\mu_k^\top\bm x-d_t\lVert\bm\mu_k\rVert_2^2,
\]
we obtain from~\eqref{eq:coefficient_approx}
\[
\begin{aligned}
|\widehat b_k(\bm x,t)-b_k(\bm x,t)|
&\le
|\widehat c(t)-c_t||\bm\mu_k^\top\bm x|
+|\widehat d(t)-d_t|\lVert\bm\mu_k\rVert_2^2\\
&\le
C(1+s_-^{-2})\delta
\bigl(\mu_{\max}\lVert\bm x\rVert_2+\mu_{\max}^2\bigr)\\
&\le
C_{\sigma,\mu}\delta(1+\lVert\bm x\rVert_2).
\end{aligned}
\]
Therefore,
\begin{equation}
\label{eq:affine_logit_pointwise}
\max_{1\le k\le K}
|\widehat b_k(\bm x,t)-b_k(\bm x,t)|
\le
C_{\sigma,\mu}\delta(1+\lVert\bm x\rVert_2).
\end{equation}

Set \(\widehat y_k = \widehat b_k-\max_j\widehat b_j, y_k = b_k-\max_j b_j,\)
and \(\widehat{\widetilde y}_k = \max\{\widehat y_k,-M\}, \widetilde y_k = \max\{y_k,-M\}.\) The maximum map is \(1\)-Lipschitz in \(\ell^\infty\), and scalar clipping is \(1\)-Lipschitz. Hence
\[
|\widehat{\widetilde y}_k-\widetilde y_k|
\le
2\max_j|\widehat b_j-b_j|
\le
C_{\sigma,\mu}\delta(1+\lVert\bm x\rVert_2).
\]
Taking \(L^p(p_t)\)-norms and applying~\autoref{lem:gmm_uniform_moment} proves~\eqref{eq:affine_logit_Lp}.

For every \((\bm x,t)\), at least one index \(k_*\) satisfies \(\widehat b_{k_*}=\max_j\widehat b_j\). Thus \(\widehat y_{k_*}=0\), and clipping at level \(-M\) leaves this coordinate unchanged. Therefore \(\max_k\widehat{\widetilde y}_k=0\) exactly.

It remains to verify the network complexity. \autoref{lem:scalar_ou_coefficients} gives depth, width, and sparsity bounded by \(C(1+\log A_{\delta,\sigma})^2\), \(C(1+\log A_{\delta,\sigma})^3\), and \(C(1+\log A_{\delta,\sigma})^4\), respectively, and parameter bound at most \(C(1+\delta^{-2}+\sigma^4+\sigma^{-4})\). The exact square and multiplication modules add only constant depth and constant sparsity per scalar operation. The \(K\) affine logits use \(O(dK)\) nonzero parameters. While the scalar coefficient subnetworks are evaluated, the \(d\) coordinates of \(\bm x\) are propagated by the exact identity modules of~\autoref{lem:exact_relu_requ_modules}; after the logits are formed, their \(K\) coordinates are propagated through the balanced maximum tree and the final clipping layers. \autoref{lem:parallel_identity_padding} bounds the identity and padding cost by
\[
C\bigl[dQ_\delta+K\bigl(Q_\delta+\log K\bigr)\bigr]
\le
C(d+K)\Bigl(\log K+\bigl(1+\log A_{\delta,\sigma}\bigr)^4\Bigr).
\]
Together with the \(O(dK)\) affine cost and the scalar-network sparsity, this proves the stated sparsity bound. The balanced maximum tree has depth \(O(\log K)\) and width \(O(K)\), so the displayed depth and width bounds also follow. Finally,~\autoref{lem:scalar_ou_coefficients} gives
\[
B_\delta
\le
C\bigl(1+\delta^{-2}+\sigma^4+\sigma^{-4}\bigr).
\]
The affine-logit coefficients are bounded by \(\mu_{\max},\mu_{\max}^2\), and \(\log(\pi_{\min}^{-1})\), while clipping uses the bias \(M\). Parallelization and concatenation preserve the maximum parameter magnitude up to an absolute constant; they do not multiply the parameter bounds of the constituent modules. The exact square, product, maximum, clipping, and identity modules use only constant-size coefficients. Taking the maximum of these modulewise bounds proves the displayed estimate for \(B_{\mathrm{logit}}\).
\end{proof}

\subsection{Softmax approximation on a compact set}

\begin{lemma}
\label{lem:softmax_lipschitz_infty_one}
For all \(\bm u,\bm v\in\mathbb R^K\),
\begin{equation}
\label{eq:softmax_lipschitz_infty_one}
\|\operatorname{softmax}(\bm u)-\operatorname{softmax}(\bm v)\|_1
\le 2\lVert\bm u-\bm v\rVert_\infty.
\end{equation}
\end{lemma}

\begin{proof}
Let \(S=\operatorname{softmax}\). For \(\bm h\in\mathbb R^K\), \([DS(\bm z)\bm h]_i
=S_i(\bm z)\Bigl(h_i-\sum_jS_j(\bm z)h_j\Bigr).\)
Consequently,
\begin{equation*}
\lVert DS(\bm z)\bm h\rVert_1
\le
\sum_iS_i(\bm z)\Biggl(|h_i|+\sum_jS_j(\bm z)|h_j|\Biggr)
\le2\lVert\bm h\rVert_\infty.
\end{equation*}
Integrating the derivative along the line segment from \(\bm u\) to \(\bm v\) proves~\eqref{eq:softmax_lipschitz_infty_one}.
\end{proof}

\begin{lemma}
\label{lem:requ_softmax_approx}
Let \(K\ge2\), \(M>0\), and \(0<\xi<1\). There exists a ReLU--ReQU network
\[
\mathcal N_{\mathrm{softmax}}
\colon[-M,0]^K\to\mathbb R^K
\]
such that
\begin{equation}
\label{eq:softmax_network_uniform}
\sup_{\substack{\bm y\in[-M,0]^K\\ \max_jy_j=0}}
\|
\mathcal N_{\mathrm{softmax}}(\bm y)
-
\operatorname{softmax}(\bm y)
\|_1
\le\xi.
\end{equation}
The network may be chosen such that
\begin{equation*}
\begin{aligned}
Q_{\mathrm{sm}}
&\le C\log^2(K/\xi),\\
G_{\mathrm{sm}}
&\le C\bigl(K\log^2(K/\xi)+\log^3(K/\xi)\bigr),\\
S_{\mathrm{sm}}
&\le C \bigl(K\log^3(K/\xi)+\log^4(K/\xi)\bigr),
\end{aligned}
\end{equation*}
with parameter bound
\begin{equation*}
B_{\mathrm{sm}}
\le
C\bigl(1+K^4\xi^{-2}\bigr)
\le
C(K+\xi^{-1})^6.
\end{equation*}
Moreover, all accuracy-dependent subnetworks are ReLU networks, and the construction uses only one ReQU activation layer, which realizes the final \(K\) products in parallel.
\end{lemma}

\begin{proof}
Set \(\eta \coloneqq \frac{\xi}{K(4K+3)}.\) Apply Lemma~B.1 in~\cite{oko2023diffusion} with the constant schedule \(\beta_t\equiv1\). Since the corresponding coefficient is \(m_t=e^{-t}\), there exists a one-dimensional ReLU network \(\phi_{\exp,\eta}\) such that
\[
\sup_{t\ge0}
|\phi_{\exp,\eta}(t)-e^{-t}|
\le\eta,
\]
with
\[
Q_{\exp},G_{\exp}\le C\log^2(\eta^{-1}),
\qquad
S_{\exp}\le C\log^3(\eta^{-1}),
\qquad
B_{\exp}\le C\log(\eta^{-1}).
\]
Define \(E_\eta(u)=\phi_{\exp,\eta}(-u)\) for \(u\le0\). Then \(\sup_{u\in[-M,0]}|E_\eta(u)-e^u|\le\eta.\)

For \(\bm y=(y_1,\ldots,y_K)\in[-M,0]^K\), define \(\widehat u_i=E_\eta(y_i),\widehat Z=\sum_{i=1}^K\widehat u_i,\) and write \(u_i=e^{y_i}\), \(Z=\sum_i u_i\). Then \(|\widehat u_i-u_i|\le\eta,|\widehat Z-Z|\le K\eta.\) Since \(\max_i y_i=0\), one has \(1\le Z\le K\). Moreover, \(K\eta\le1/2\), and hence \(\frac12\le\widehat Z\le2K.\)

Set \(\rho \coloneqq \min\bigl\{\eta,\frac{1}{2K}\bigr\}.\) By Lemma~F.7 in~\cite{oko2023diffusion}, there exists a one-dimensional ReLU network \(\phi_{\mathrm{rec},\rho}\) such that
\[
\sup_{z\in[\rho,\rho^{-1}]}
|\phi_{\mathrm{rec},\rho}(z)-z^{-1}|
\le\rho.
\]
Since \([1/2,2K]\subset[\rho,\rho^{-1}]\), the same error is at most \(\eta\) on \([1/2,2K]\). Define \(\widehat r = \phi_{\mathrm{rec},\rho}(\widehat Z), r=Z^{-1},\) and set \([\mathcal N_{\mathrm{softmax}}(\bm y)]_i = \widehat u_i\widehat r.\) The \(K\) products are realized in parallel by one exact ReQU multiplication layer.

Using \(Z\ge1\), \(\widehat Z\ge1/2\), and \(|\widehat Z-Z|\le K\eta\), we obtain
\[
|\widehat r-r|
\le
|\widehat r-\widehat Z^{-1}|
+|\widehat Z^{-1}-Z^{-1}|
\le
(2K+1)\eta.
\]
Moreover, \(|\widehat u_i|\le2\). Therefore,
\[
|\widehat u_i\widehat r-u_ir|
\le
|\widehat u_i-u_i||r|
+|\widehat u_i||\widehat r-r|
\le
(4K+3)\eta.
\]
Summing over \(i\) proves~\eqref{eq:softmax_network_uniform}.

It remains to verify the complexity. The \(K\) copies of \(E_\eta\) are evaluated in parallel, and the reciprocal network has depth, width, and sparsity polynomial in \(\log(\rho^{-1})\). Since
\[
\eta^{-1}
=
\frac{K(4K+3)}{\xi}
\le
\frac{7K^2}{\xi},
\qquad
\rho^{-1}
\le
\max\{\eta^{-1},2K\},
\]
we have
\[
\log(\eta^{-1})+\log(\rho^{-1})
\le
C\log(K/\xi),
\]
which gives the stated bounds for \(Q_{\mathrm{sm}},G_{\mathrm{sm}}\), and \(S_{\mathrm{sm}}\).

For the parameter bound, Lemma~F.7 gives
\[
B_{\mathrm{rec}}
\le
C\rho^{-2}
\le
C K^4\xi^{-2}.
\]
The exponential subnetworks satisfy \(B_{\mathrm{exp}}\le C\log(\eta^{-1})\). The summation and exact multiplication modules use only constant-size weights, as do the identity-propagation modules. Consequently,
\[
B_{\mathrm{sm}}
\le
C(1+K^4\xi^{-2})
\le
C(K+\xi^{-1})^6.
\]
The exponential and reciprocal subnetworks are pure ReLU networks, while the only ReQU activation layer is the final parallel multiplication layer.
\end{proof}

\subsection{Completion of the proof}

\begin{proof}[Proof of~\autoref{thm:score_Lp_approx}]
Recall the definitions
\[
s_-^2=\min\{1,\sigma^2\},
\qquad
s_+^2=\max\{1,\sigma^2\},
\qquad
D_p=1+\mu_{\max}+s_+(\sqrt d+\sqrt p).
\]
Enlarge the constant in~\autoref{lem:affine_logit_approx}, if necessary, so that \(C_{\sigma,\mu}\ge1\).

If \(K=1\), then \(\nabla_{\bm x}\log p_t(\bm x) = -r_t(\bm x-m_t\bm\mu_1).\) Choosing the scalar approximation accuracy \(\delta = \frac{c\varepsilon}{(1+s_-^{-2})(1+\mu_{\max})D_p}\) and applying~\autoref{lem:scalar_ou_coefficients} gives~\eqref{eq:score_Lp_target}. The resulting scalar-network bounds are covered by the estimates stated below.  

Assume \(K\ge2\). Define the posterior-weight tolerance \(\varepsilon_\gamma \coloneqq \frac{s_-^2}{16(1+\mu_{\max})}\varepsilon,\) and choose
\[
\xi
=
\frac{\varepsilon_\gamma}{3},
\qquad
M
=
\log\biggl(\frac{6K}{\varepsilon_\gamma}\biggr),
\qquad
\delta
=
\frac{\varepsilon_\gamma}
{12K^{1/p}C_{\sigma,\mu}D_p}.
\]
Since \(0<\varepsilon<1/2\), these choices satisfy \(0<\xi<1\) and \(0<\delta<1/2\).

By~\autoref{lem:affine_logit_approx},
\[
\max_k
\Bigl\|
\widehat{\widetilde y}_k(\bm X_t,t)
-
\widetilde y_k(\bm X_t,t)
\Bigr\|_{L^p(p_t)}
\le
\frac{\varepsilon_\gamma}{12K^{1/p}}
\]
uniformly in \(t\). Apply~\autoref{lem:requ_softmax_approx} with accuracy \(\xi\), and define \(\widehat{\bm\gamma}(\bm x,t) = \mathcal N_{\mathrm{softmax}} (\widehat{\widetilde{\bm y}}(\bm x,t)).\) Both the exact and approximate truncated logit vectors have maximum zero exactly. By~\autoref{lem:softmax_lipschitz_infty_one},
\[
\begin{aligned}
&\Bigl\|
\operatorname{softmax}(\widehat{\widetilde{\bm y}}(\bm X_t,t))
-
\operatorname{softmax}(\widetilde{\bm y}(\bm X_t,t))
\Bigr\|_{L^p(p_t;\ell^1)}\\
&\qquad\le
2\Biggl(
\sum_{k=1}^K
\Bigl\|
\widehat{\widetilde y}_k-
\widetilde y_k
\Bigr\|_{L^p}^p
\Biggr)^{1/p}
\le
\frac{\varepsilon_\gamma}{6}.
\end{aligned}
\]
Moreover, \(2Ke^{-M} = \frac{\varepsilon_\gamma}{3}.\) Combining~\autoref{lem:requ_softmax_approx},~\autoref{lem:softmax_lipschitz_infty_one}, and~\autoref{lem:softmax_tail} yields
\begin{equation}
\label{eq:gamma_Lp_final_revised}
\sup_{t\in[0,T]}
\|
\widehat{\bm\gamma}(\cdot,t)-\bm\gamma(\cdot,t)
\|_{L^p(p_t;\ell^1)}
\le
\varepsilon_\gamma.
\end{equation}

Let \(\widehat m_\delta\) and \(\widehat r_\delta\) be the scalar networks in~\autoref{lem:scalar_ou_coefficients}, and define
\[
\widehat{\bm s}_\varepsilon(\bm x,t)
=
-\widehat r_\delta(t)
\Biggl(
\bm x-
\widehat m_\delta(t)
\sum_{k=1}^K\widehat\gamma_k(\bm x,t)\bm\mu_k
\Biggr).
\]
Write \(\bm M_\gamma = \sum_k\gamma_k\bm\mu_k, \widehat{\bm M}_\gamma = \sum_k\widehat\gamma_k\bm\mu_k.\) The clipping in~\autoref{lem:scalar_ou_coefficients} gives \(|\widehat r_\delta(t)|\le2s_-^{-2}, 0\le\widehat m_\delta(t),m_t\le1.\) Hence
\[
|\widehat r_\delta\widehat m_\delta-r_tm_t|
\le
(1+2s_-^{-2})\delta,
\]
and
\[
\lVert
\widehat{\bm s}_\varepsilon(\bm x,t)
-
\nabla_{\bm x}\log p_t(\bm x)
\rVert_2
\le
\delta\lVert\bm x\rVert_2
+2s_-^{-2}\mu_{\max}
\lVert\widehat{\bm\gamma}-\bm\gamma\rVert_1 + (1+2s_-^{-2})\mu_{\max}\delta.
\]
Taking \(L^p(p_t)\)-norms, applying~\autoref{lem:gmm_uniform_moment} and~\eqref{eq:gamma_Lp_final_revised}, and using the definitions of \(\delta\) and \(\varepsilon_\gamma\), we obtain
\[
\sup_{t\in[0,T]}
(
\mathbb E_{\bm X_t\sim p_t}
[
\lVert
\widehat{\bm s}_\varepsilon(\bm X_t,t)
-
\nabla_{\bm x}\log p_t(\bm X_t)
\rVert_2^p
]
)^{1/p}
\le
C\varepsilon.
\]
This proves~\eqref{eq:score_Lp_target}.

It remains to verify the complexity bounds. Since
\[
\varepsilon_\gamma^{-1}
\le
C(1+\mu_{\max})s_-^{-2}\varepsilon^{-1},
\qquad
C_{\sigma,\mu}
\le
C(1+s_-^{-2})(1+\mu_{\max})^2,
\]
one has
\[
\delta^{-1}
\le
C K^{1/p}D_p s_-^{-2}(1+s_-^{-2})
(1+\mu_{\max})^3\varepsilon^{-1}.
\]
Let \(N_{\varepsilon,\sigma}=2+d+K+\varepsilon^{-1}+\sigma+\sigma^{-1}\). If \(0<\sigma\le1\), then \(D_p\le C N_{\varepsilon,\sigma}^{1/2}\) and \(s_-^{-2}(1+s_-^{-2})\le C\sigma^{-4}\le C N_{\varepsilon,\sigma}^4\). If \(\sigma\ge1\), then \(D_p\le C N_{\varepsilon,\sigma}^{3/2}\) and \(s_-^{-2}(1+s_-^{-2})=2\). Therefore,
\[
\delta^{-1}
\le
\begin{cases}
C N_{\varepsilon,\sigma}^{11/2+1/p},&0<\sigma\le1,\\
C N_{\varepsilon,\sigma}^{5/2+1/p},&\sigma\ge1,
\end{cases}
\qquad
\text{and hence}
\qquad
\delta^{-1}
\le
C N_{\varepsilon,\sigma}^{11/2+1/p}.
\]
Moreover,
\[
\xi^{-1}
\le
C(1+\mu_{\max})s_-^{-2}\varepsilon^{-1}
\le
C N_{\varepsilon,\sigma}^{3},
\qquad
M\le C\bigl(1+\log N_{\varepsilon,\sigma}\bigr).
\]
Consequently,
\[
\log(\delta^{-1})+\log(\xi^{-1})+M
\le
C\bigl(1+\log N_{\varepsilon,\sigma}\bigr).
\]
The scalar coefficient networks, the exact product modules, the \(K\) parallel affine logits, the maximum tree, clipping, the \(K\) parallel exponential subnetworks, and the reciprocal normalization therefore give
\[
Q
\le
C\Bigl(
\log K+\bigl(1+\log N_{\varepsilon,\sigma}\bigr)^2
\Bigr),
\qquad
G
\le
CK\Bigl(
 d+\bigl(1+\log N_{\varepsilon,\sigma}\bigr)^3
\Bigr).
\]
For sparsity, the affine logits and the final posterior mean contribute \(O(Kd)\) nonzero parameters. \autoref{lem:affine_logit_approx} accounts for propagation of the \(d\) input coordinates through the coefficient subnetworks and for preservation of the \(K\) logits through the maximum tree. \autoref{lem:requ_softmax_approx} contributes \(O(K(1+\log N_{\varepsilon,\sigma})^3+(1+\log N_{\varepsilon,\sigma})^4)\) nonzero parameters. Applying~\autoref{lem:parallel_identity_padding} to the raw input, logit, posterior, and scalar-coefficient channels gives an additional cost bounded by \(C(d+K)(\log K+(1+\log N_{\varepsilon,\sigma})^2)\), which is dominated by the quartic logarithmic envelope below. Hence
\[
S
\le
C\Bigl[
Kd+(d+K)\Bigl(\log K+\bigl(1+\log N_{\varepsilon,\sigma}\bigr)^4\Bigr)
\Bigr].
\]

For the parameter bound,~\autoref{lem:affine_logit_approx} and the estimate for \(\delta^{-1}\) give
\[
B_{\mathrm{logit}}
\le
C\bigl(
1+M+\delta^{-2}+\sigma^4+\sigma^{-4}
\bigr)
\le
C N_{\varepsilon,\sigma}^{11+2/p}.
\]
Here the last inequality follows by using \(\delta^{-2}\le C N_{\varepsilon,\sigma}^{11+2/p}\) when \(\sigma\le1\), while for \(\sigma\ge1\) one has \(\delta^{-2}\le C N_{\varepsilon,\sigma}^{5+2/p}\) and \(\sigma^4\le N_{\varepsilon,\sigma}^4\). \autoref{lem:requ_softmax_approx} and \(\xi^{-1}\le C N_{\varepsilon,\sigma}^{3}\) give the sharper softmax estimate
\[
B_{\mathrm{sm}}
\le
C\bigl(1+K^4\xi^{-2}\bigr)
\le
C N_{\varepsilon,\sigma}^{10}.
\]
All exact affine, maximum, clipping, square, multiplication, and identity modules use parameter magnitudes bounded by the same envelope. Since \(11+2/p>10\), the complete score network satisfies
\[
B
\le
C N_{\varepsilon,\sigma}^{11+2/p}
\le
C N_{\varepsilon,\sigma}^{13},
\]
where the last inequality uses \(p\ge1\). This proves all the stated complexity bounds.
\end{proof}

\section[Proof of the stability-certificate mismatch theorem]{Proof of~\autoref{thm:lipschitz_mismatch_distill}}
\label{app:proof_lipschitz_mismatch_distill}
This section proves~\autoref{thm:lipschitz_mismatch_distill}. The proof has two parts. First, the mixture term in \(L(t)\) gives a closed-form lower bound on the tail stability exponent. Second, bounded inputs and bounded network parameters, together with a fixed number of ReQU activation layers, give an explicit local Lipschitz budget without requiring an additional preactivation-range assumption.

\begin{proof}[Proof of~\autoref{thm:lipschitz_mismatch_distill}]
The mixture contribution in~\eqref{eq:L_profile_distill} gives
\begin{equation*}
L(t)
\ge
\frac{\beta D_0^2}{4}
\frac{e^{-2\beta t}}
{\bigl(1-(1-\sigma^2)e^{-2\beta t}\bigr)^2}.
\end{equation*}
With \(\Delta=(\log 2)/(2\beta)\) and the substitution \(u=e^{-2\beta t}\),
\begin{align*}
\Lambda_{\mathrm{tail}}
&\ge
\frac{\beta D_0^2}{4}
\int_0^\Delta
\frac{e^{-2\beta t}}
{\bigl(1-(1-\sigma^2)e^{-2\beta t}\bigr)^2}
\mathrm dt\\
&=
\frac{D_0^2}{8}
\int_{1/2}^{1}
\frac{\mathrm du}{\bigl(1-(1-\sigma^2)u\bigr)^2}\\
&=
\frac{D_0^2}{8(1-\sigma^2)}
\biggl(
\frac1{\sigma^2}-\frac2{1+\sigma^2}
\biggr)
=
\frac{D_0^2}{8\sigma^2(1+\sigma^2)}.
\end{align*}
The last identity extends continuously to \(\sigma=1\). If \(\sigma\le1\), then \(1+\sigma^2\le2\), proving
\begin{equation}
\label{eq:Lambda_tail_simple_lower_revised}
\Lambda_{\mathrm{tail}}
\ge
\frac{D_0^2}{16\sigma^2}.
\end{equation}

We next bound the local student budget. Let \(\bm f\in\mathcal F_{d\to d}^{(r_{\mathrm{ReQU}})}(Q,G,S,B)\), and suppose that it has \(L\le Q\) activation layers. All input, hidden, and output dimensions are at most \(\overline G=\max\{d,G\}\). Write \(\bm h_0(\bm x)=\bm x\), and for \(\ell=1,\ldots,L\), let
\begin{equation*}
\bm z_\ell(\bm x)
=
\bm W_{\ell-1}\bm h_{\ell-1}(\bm x)+\bm b_\ell,
\qquad
\bm h_\ell(\bm x)
=
\varrho_\ell\bigl(\bm z_\ell(\bm x)\bigr),
\end{equation*}
where each \(\varrho_\ell\) is either ReLU or ReQU. Let \(n_\ell^{\mathrm{ReQU}}\) denote the number of ReQU activation layers among the first \(\ell\) layers, so that \(n_\ell^{\mathrm{ReQU}}\le r_{\mathrm{ReQU}}\), and define
\begin{equation*}
H_\ell
\coloneqq
\max\Biggl\{
1,
\sup_{\bm x\in B_R}
\lVert\bm h_\ell(\bm x)\rVert_\infty
\Biggr\}.
\end{equation*}
By the parameter bound and the dimension bound,
\begin{equation*}
\sup_{\bm x\in B_R}
\lVert\bm z_\ell(\bm x)\rVert_\infty
\le
B\overline G H_{\ell-1}+B
\le
B(\overline G+1)H_{\ell-1}
\le
\Gamma_R H_{\ell-1},
\end{equation*}
where \(\Gamma_R = \max\bigl\{ e, C(1+B\overline G)(1+R) \bigr\}\) and \(C\ge1\) is chosen sufficiently large. Since \(H_0\le1+R\le\Gamma_R\), a direct induction gives
\begin{equation}
\label{eq:hidden_range_fixed_requ}
H_\ell
\le
\Gamma_R^{2^{n_\ell^{\mathrm{ReQU}}}(\ell+1)}
\le
\Gamma_R^{2^{r_{\mathrm{ReQU}}}(Q+1)},
\qquad
\ell=0,\ldots,L.
\end{equation}
A ReLU layer satisfies \(H_\ell\le\Gamma_R H_{\ell-1}\). A ReQU layer satisfies \(H_\ell\le(\Gamma_R H_{\ell-1})^2\). Thus each ReQU layer doubles the exponent, and this occurs at most \(r_{\mathrm{ReQU}}\) times.

Each affine matrix satisfies \(\lVert\bm W_\ell\rVert_2 \le \lVert\bm W_\ell\rVert_F \le B\overline G \le \Gamma_R.\) ReLU is \(1\)-Lipschitz. For a ReQU layer, \(\sigma_2(z)=(z_+)^2\) and \(\sigma_2'(z)=2z_+\), so~\eqref{eq:hidden_range_fixed_requ} gives
\begin{equation*}
\operatorname{Lip}\bigl(\sigma_2;\bm z_\ell(B_R)\bigr)
\le
2\sup_{\bm x\in B_R}
\lVert\bm z_\ell(\bm x)\rVert_\infty
\le
\Gamma_R^{(1+2^{r_{\mathrm{ReQU}}})(Q+1)}.
\end{equation*}
Since at most \(r_{\mathrm{ReQU}}\) activation layers are ReQU, Lipschitz submultiplicativity yields
\begin{align*}
\operatorname{Lip}(\bm f;B_R)
&\le
\Gamma_R^{L+1}
\Bigl(
\Gamma_R^{(1+2^{r_{\mathrm{ReQU}}})(Q+1)}
\Bigr)^{r_{\mathrm{ReQU}}}
\le
\Gamma_R^{c_{r_{\mathrm{ReQU}}}(Q+1)},
\end{align*}
where \(c_{r_{\mathrm{ReQU}}} \coloneqq 1+r_{\mathrm{ReQU}}(1+2^{r_{\mathrm{ReQU}}}).\) Taking the supremum over \(\mathcal F_{d\to d}^{(r_{\mathrm{ReQU}})}(Q,G,S,B)\) proves~\eqref{eq:student_local_budget_corrected}.

Finally, suppose that \(\sigma\le1\) and condition~\eqref{eq:certificate_mismatch_condition} holds. Combining this condition with~\eqref{eq:Lambda_tail_simple_lower_revised} gives
\begin{equation*}
\exp(\Lambda_{\mathrm{tail}})
>
\Gamma_R^{c_{r_{\mathrm{ReQU}}}(Q+1)}
\ge
\mathrm{LipBudget}^{(r_{\mathrm{ReQU}})}_R.
\end{equation*}
This completes the proof of~\autoref{thm:lipschitz_mismatch_distill}.
\end{proof}

\section[Proof of the residual flow-approximation theorem]{Proof of~\autoref{thm:resnet_flow_tail_approx}}
\label{appsec:proof_resnet_flow_tail_approx}

The proof combines a compact-uniform fixed-time score construction, a global first-order consistency estimate, and localization of the exact probability-flow trajectory.

\begin{lemma}
\label{lem:compact_uniform_score_approx}
Assume~\autoref{ass:gmm_ou_init}. For every \(R\ge1\), \(0<\eta<1\), and \(t\in[0,T]\), there exists a time-independent ReLU--ReQU network \(\widehat{\bm s}_{t,R,\eta}\colon\mathbb R^d\to\mathbb R^d\) such that
\begin{equation}
\label{eq:compact_uniform_score_bound}
\sup_{\lVert\bm x\rVert_2\le R}
\lVert
\widehat{\bm s}_{t,R,\eta}(\bm x)
-
\nabla_{\bm x}\log p_t(\bm x)
\rVert_2
\le\eta.
\end{equation}
Uniformly in \(t\), it may be chosen with
\begin{equation}
\label{eq:compact_score_complexity}
\begin{aligned}
Q&\le C_1\Bigl(\log K+\bigl(1+\log(2+K+\eta^{-1}+\sigma+\sigma^{-1})\bigr)^2\Bigr),\\
G&\le C_2K\Bigl(d+\bigl(1+\log(2+K+\eta^{-1}+\sigma+\sigma^{-1})\bigr)^3\Bigr),\\
S&\le C_3\Bigl[
Kd+(d+K)\Bigl(\log K+\bigl(1+\log(2+K+\eta^{-1}+\sigma+\sigma^{-1})\bigr)^4\Bigr)
\Bigr],\\
B&\le C_4\bigl(2+K+R+\eta^{-1}+\sigma+\sigma^{-1}\bigr)^{10}.
\end{aligned}
\end{equation}
The constants depend only on \(\pi_{\min}\) and the mixture means and weights, but not on \(R,\eta,t\), or \(\sigma\).
\end{lemma}

\begin{proof}
Write
\[
s_-^2=\min\{1,\sigma^2\}.
\]
If \(K=1\), then the score is affine \(\nabla_{\bm x}\log p_t(\bm x) = -s_t^{-2}(\bm x-m_t\bm\mu_1).\) It is represented exactly by a constant-depth network whose parameter bound is at most \(C(1+\sigma^{-2})(1+\mu_{\max})\), so the conclusion follows. Assume henceforth that \(K\ge2\).

For fixed \(t\), use the affine logits in~\eqref{eq:affine_logit_def}. For \(\lVert\bm x\rVert_2\le R\), their pairwise differences satisfy
\[
|b_k(\bm x,t)-b_j(\bm x,t)|
\le
C\bigl[
\log(\pi_{\min}^{-1})
+s_-^{-2}(\mu_{\max}R+\mu_{\max}^2)
\bigr].
\]
Choose \(M = C\bigl[ 1+\log(\pi_{\min}^{-1}) +s_-^{-2}(\mu_{\max}R+\mu_{\max}^2) \bigr].\) Then clipping of the stabilized logits is inactive on \(B_R\). At fixed time, each \(b_k(\cdot,t)\) is affine and is represented exactly, and the maximum tree is exact.

Set \(\xi = \min\Bigl\{ \frac12, \frac{\eta s_-^2}{2(1+\mu_{\max})} \Bigr\}.\)~\autoref{lem:requ_softmax_approx} gives a posterior approximation \(\widehat{\bm\gamma}\) satisfying
\[
\sup_{\lVert\bm x\rVert_2\le R}
\lVert\widehat{\bm\gamma}(\bm x)-\bm\gamma(\bm x,t)\rVert_1
\le\xi.
\]
Using the exact coefficients \(m_t\) and \(s_t^{-2}\) as fixed network weights, define
\[
\widehat{\bm s}_{t,R,\eta}(\bm x)
=
-s_t^{-2}
\Biggl(
\bm x-m_t\sum_{k=1}^K\widehat\gamma_k(\bm x)\bm\mu_k
\Biggr).
\]
Then
\[
\begin{aligned}
\lVert
\widehat{\bm s}_{t,R,\eta}(\bm x)
-
\nabla_{\bm x}\log p_t(\bm x)
\rVert_2
&\le
s_t^{-2}m_t\mu_{\max}
\lVert\widehat{\bm\gamma}(\bm x)-\bm\gamma(\bm x,t)\rVert_1\\
&\le
s_-^{-2}\mu_{\max}\xi
\le
\eta,
\end{aligned}
\]
which proves~\eqref{eq:compact_uniform_score_bound}.

It remains to verify the complexity. Since \(\xi^{-1} \le C(1+\mu_{\max})(1+\sigma^{-2})\eta^{-1},\)~\autoref{lem:requ_softmax_approx} gives the displayed depth and width bounds. The two affine stages, namely the \(K\) logits and the final posterior mean, contribute \(O(Kd)\) nonzero parameters. Exact propagation of the \(d\) input coordinates and padding of the \(K\) logit and posterior channels across the maximum and softmax subnetworks contribute at most
\[
C(d+K)\Bigl(\log K+\bigl(1+\log(2+K+\eta^{-1}+\sigma+\sigma^{-1})\bigr)^4\Bigr),
\]
by~\autoref{lem:exact_relu_requ_modules},~\autoref{lem:parallel_identity_padding}, and~\autoref{lem:requ_softmax_approx}. This proves the stated sparsity bound. For the parameter bound, first note that
\[
s_-^{-2}
\le
(1+\sigma^{-1})^2
\le
\bigl(2+K+R+\eta^{-1}+\sigma+\sigma^{-1}\bigr)^2.
\]
Hence the exact coefficients in the affine logits and in the final score map have magnitude at most the second power of the displayed envelope, while
\[
M
\le
C\bigl(2+K+R+\eta^{-1}+\sigma+\sigma^{-1}\bigr)^3.
\]
Moreover,
\[
\xi^{-1}
\le
C(1+\mu_{\max})s_-^{-2}\eta^{-1}
\le
C\bigl(2+K+R+\eta^{-1}+\sigma+\sigma^{-1}\bigr)^3.
\]
Using the sharper parameter estimate from~\autoref{lem:requ_softmax_approx},
\[
B_{\mathrm{sm}}
\le
C\bigl(1+K^4\xi^{-2}\bigr)
\le
C\bigl(2+K+R+\eta^{-1}+\sigma+\sigma^{-1}\bigr)^{10}.
\]
This dominates the affine coefficients, the clipping bias, and the constant-size exact modules, and proves the stated parameter bound. This completes the proof.
\end{proof}

\begin{lemma}
\label{lem:one_step_consistency_global}
Under the assumptions of~\autoref{thm:resnet_flow_tail_approx}, there are constants \(C_v,\overline L<\infty\) such that, for all \((\bm x,t)\in\mathbb R^d\times[0,T]\),
\begin{equation}
\label{eq:global_v_growth}
\lVert\bm v(\bm x,t)\rVert_2
+
\lVert\partial_t\bm v(\bm x,t)\rVert_2
\le C_v(1+\lVert\bm x\rVert_2),
\qquad
\lVert\nabla_{\bm x}\bm v(\bm x,t)\rVert_2
\le\overline L.
\end{equation}
Consequently, for \(0\le t-h<t\le T\), \(0<h\le1\), and every \(\bm x\in\mathbb R^d\),
\begin{equation}
\label{eq:one_step_consistency_exact}
\bigl\lVert
\Phi_{t-h\leftarrow t}(\bm x)
-
\bigl(\bm x-h\bm v(\bm x,t)\bigr)
\bigr\rVert_2
\le
C(1+\lVert\bm x\rVert_2)h^2.
\end{equation}
\end{lemma}

\begin{proof}
The score formula immediately gives \(\lVert\nabla\log p_t(\bm x)\rVert_2 \le s_-^{-2}(\lVert\bm x\rVert_2+\mu_{\max}),\) so \(\bm v\) has linear growth. The spatial derivative is globally bounded by \(L(t)\), and \(\overline L\coloneqq\sup_{t\in[0,T]}L(t)<\infty\) because \(\sigma>0\). For the time derivative, write the posterior using the affine logits \(b_k\) in~\eqref{eq:affine_logit_def}. Under the constant schedule, the scalar coefficients in \(b_k\) are continuously differentiable and uniformly bounded together with their derivatives. Hence \(|\partial_t b_k(\bm x,t)|\le C(1+\lVert\bm x\rVert_2).\) The softmax identity \(\partial_t\gamma_k = \gamma_k\Bigl( \partial_t b_k- \sum_j\gamma_j\partial_t b_j \Bigr)\) implies \(\sum_k|\partial_t\gamma_k|\le C(1+\lVert\bm x\rVert_2)\). Differentiating
\(\overline{\bm m}=m_t\sum_k\gamma_k\bm\mu_k\) and then \(\nabla\log p_t=-s_t^{-2}(\bm x-\overline{\bm m})\) proves the time-growth estimate in~\eqref{eq:global_v_growth}.

Let \(\bm z_u=\Phi_{u\leftarrow t}(\bm x)\), \(u\in[t-h,t]\). Linear growth and Gr\"onwall's inequality give
\begin{equation*}
1+\lVert\bm z_u\rVert_2
\le C(1+\lVert\bm x\rVert_2),
\qquad
\lVert\bm z_u-\bm x\rVert_2
\le C(t-u)(1+\lVert\bm x\rVert_2).
\end{equation*}
Using the integral representation of the backward segment,
\begin{align*}
&\bigl\lVert
\Phi_{t-h\leftarrow t}(\bm x)-\bigl(\bm x-h\bm v(\bm x,t)\bigr)
\bigr\rVert_2\\
&\quad\le
\int_{t-h}^{t}
\lVert\bm v(\bm z_u,u)-\bm v(\bm z_u,t)\rVert_2\mathrm du
+
\int_{t-h}^{t}
\lVert\bm v(\bm z_u,t)-\bm v(\bm x,t)\rVert_2\mathrm du\\
&\quad\le
C(1+\lVert\bm x\rVert_2)
\int_{t-h}^{t}(t-u)\mathrm du
+
\overline L C(1+\lVert\bm x\rVert_2)
\int_{t-h}^{t}(t-u)\mathrm du,
\end{align*}
which is~\eqref{eq:one_step_consistency_exact}.
\end{proof}

\begin{lemma}
\label{lem:trajectory_localization_tail}
Let \(\bm X_T\sim p_T\) and \(\bm X_t^\star=\Phi_{t\leftarrow T}(\bm X_T)\). There are constants \(C,c>0\) such that, for all sufficiently large \(R\),
\begin{equation}
\label{eq:trajectory_tail_revised}
\mathbb P\Biggl(
\sup_{t\in[0,T]}\lVert\bm X_t^\star\rVert_2>R
\Biggr)
\le Ce^{-cR^2}.
\end{equation}
\end{lemma}

\begin{proof}
By~\autoref{prop:flow_stability}, \(\lVert\Phi_{t\leftarrow T}(\bm x)\rVert_2 \le \lVert\Phi_{t\leftarrow T}(0)\rVert_2 +e^{\Lambda_T}\lVert\bm x\rVert_2.\) For the reference trajectory \(\bm z_t=\Phi_{t\leftarrow T}(0)\), \(\lVert\bm v(0,u)\rVert_2 =\beta\lVert\nabla\log p_u(0)\rVert_2 \le\beta s_-^{-2}\mu_{\max}.\) Using \(\lVert\bm v(\bm z,u)\rVert_2\le L(u)\lVert\bm z\rVert_2+\lVert\bm v(0,u)\rVert_2\) in the backward integral equation and applying Gr\"onwall gives
\begin{equation*}
\sup_{t\in[0,T]}\lVert\Phi_{t\leftarrow T}(0)\rVert_2
\le Ce^{\Lambda_T}.
\end{equation*}
Therefore,
\begin{equation*}
\sup_{t\in[0,T]}\lVert\bm X_t^\star\rVert_2
\le Ce^{\Lambda_T}(1+\lVert\bm X_T\rVert_2).
\end{equation*}
The terminal marginal \(p_T\) is a finite Gaussian mixture, so its Euclidean norm has a Gaussian tail. This proves~\eqref{eq:trajectory_tail_revised} after changing constants.
\end{proof}

\begin{proof}[Proof of~\autoref{thm:resnet_flow_tail_approx}]
Set \(q=2p\). Let
\begin{equation*}
R=R_\varepsilon
=c_R\sqrt{1+\log(\varepsilon^{-1})},
\qquad
\rho_R=\sqrt d(R+1),
\end{equation*}
where \(c_R\) will be chosen below. Set \(\eta=\varepsilon\) and
\begin{equation}
\label{eq:n_choice_resnet_proof}
n
=
\biggl\lceil
\frac{T(1+\rho_R)}{\varepsilon}
\biggr\rceil,
\qquad
h=\frac{T}{n}
\le
\frac{\varepsilon}{1+\rho_R}
<1.
\end{equation}
Take the uniform decreasing grid \(T=t_0>t_1>\cdots>t_n=0\), and write
\begin{equation*}
\Lambda_k=\int_{t_k}^{t_{k-1}}L(u)\mathrm du,
\qquad
\sum_{k=1}^n\Lambda_k=\Lambda_T.
\end{equation*}
For each \(k\), apply~\autoref{lem:compact_uniform_score_approx} at time \(t_{k-1}\) on \(B_{\rho_R}\) to obtain
\begin{equation}
\label{eq:block_score_uniform_revised}
\sup_{\lVert\bm x\rVert_2\le\rho_R}
\bigl\lVert
\widehat{\bm s}_k(\bm x)-\nabla\log p_{t_{k-1}}(\bm x)
\bigr\rVert_2
\le\eta.
\end{equation}
Let
\begin{equation*}
\widehat{\bm v}_k(\bm x)
=-\beta\bigl(\bm x+\widehat{\bm s}_k(\bm x)\bigr),
\qquad
\Psi_k(\bm x)
=
\operatorname{Clip}_{R+1}
\bigl(\bm x-h\widehat{\bm v}_k(\bm x)\bigr).
\end{equation*}
Define
\begin{equation*}
\bm X_{t_0}^\star=\bm X_T,
\quad
\bm X_{t_k}^\star=\Phi_{t_k\leftarrow t_{k-1}}(\bm X_{t_{k-1}}^\star),
\qquad
\widehat{\bm X}_0=\bm X_T,
\quad
\widehat{\bm X}_k=\Psi_k(\widehat{\bm X}_{k-1}),
\end{equation*}
and \(\bm e_k=\widehat{\bm X}_k-\bm X_{t_k}^\star\). Consider \(\mathcal G_R = \Bigl\{ \sup_{t\in[0,T]}\lVert\bm X_t^\star\rVert_2\le R \Bigr\}.\) On \(\mathcal G_R\), every exact grid point lies in \([-R,R]^d\). Moreover, \(\lVert\widehat{\bm X}_0\rVert_2 = \lVert\bm X_T\rVert_2 \le R \le \rho_R.\) Coordinatewise clipping is the Euclidean projection onto \([-R-1,R+1]^d\), so it cannot increase the distance to an exact grid point. For \(k\ge1\), every clipped numerical iterate lies in \(B_{\rho_R}\). Hence~\eqref{eq:block_score_uniform_revised} applies at the input of every residual block.

Add and subtract
\(\Phi_{t_k\leftarrow t_{k-1}}(\widehat{\bm X}_{k-1})\). The segment stability estimate,~\autoref{lem:one_step_consistency_global}, and~\eqref{eq:block_score_uniform_revised} give, on \(\mathcal G_R\),
\begin{equation}
\label{eq:segment_recursion_revised}
\lVert\bm e_k\rVert_2
\le
e^{\Lambda_k}\lVert\bm e_{k-1}\rVert_2
+h\beta\eta
+C(1+\rho_R)h^2.
\end{equation}
Indeed, the last term is valid without any prior assumption that the exact segment from the numerical iterate remains inside a prescribed ball; this is precisely the role of the global estimate~\eqref{eq:one_step_consistency_exact}.

Iterating~\eqref{eq:segment_recursion_revised}, using \(\bm e_0=0\) and \(\sum_k\Lambda_k=\Lambda_T\), yields
\begin{align}
\lVert\bm e_n\rVert_2\mathbf1_{\mathcal G_R}
&\le
e^{\Lambda_T}
\sum_{k=1}^n
\bigl(h\beta\eta+C(1+\rho_R)h^2\bigr)\notag\\
&\le
e^{\Lambda_T}
(\beta T\eta+C(1+\rho_R)Th)
\le Ce^{\Lambda_T}\varepsilon.
\label{eq:good_event_error_revised}
\end{align}

On \(\mathcal G_R^c\), the approximate endpoint is deterministically bounded by \(\rho_R\). Moreover, the exact endpoint has law \(p_0\), because the exact probability-flow ODE preserves the marginal family, and therefore has finite \(q\)-th moment. H\"older's inequality and~\autoref{lem:trajectory_localization_tail} imply
\begin{equation}
\bigl\|\bm e_n\mathbf1_{\mathcal G_R^c}\bigr\|_{L^p}
\le
\rho_R\mathbb P(\mathcal G_R^c)^{1/p}
+
\lVert\bm X_{t_n}^\star\rVert_{L^q}
\mathbb P(\mathcal G_R^c)^{1/p-1/q}
\le
C(1+R)e^{-cR^2}.
\label{eq:bad_event_revised}
\end{equation}
Choose \(c_R\) large enough that the last expression is at most \(C\varepsilon\). Combining~\eqref{eq:good_event_error_revised} and~\eqref{eq:bad_event_revised}, and using \(e^{\Lambda_T}\ge1\), proves~\eqref{eq:resnet_flow_error_corrected}.

Finally, \(R_\varepsilon\asymp\sqrt{1+\log(\varepsilon^{-1})}\) implies \(1+\rho_R \le C\sqrt{1+\log(\varepsilon^{-1})},\) so~\eqref{eq:n_choice_resnet_proof} gives~\eqref{eq:block_count_corrected}. For the parameter magnitude, the localization radius satisfies
\[
\rho_R
\le
C\Bigl(\sqrt d+\sqrt{d\log(\varepsilon^{-1})}\Bigr)
\le
C\bigl(d+\varepsilon^{-1}\bigr)
\le
C N_{\varepsilon,\sigma}.
\]
Applying~\eqref{eq:compact_score_complexity} with radius \(\rho_R\) and accuracy \(\eta=\varepsilon\) therefore gives \(B\le C N_{\varepsilon,\sigma}^{10}\), together with the stated per-block depth, width, and sparsity bounds. The clipping map adds only constant depth and \(O(d)\) nonzero ReLU parameters. The proof is complete.
\end{proof}

\bibliographystyle{plainnat}    
\bibliography{references}    
\end{document}